\documentclass[nohyperref]{article}
\usepackage[accepted]{icml2022}

\usepackage{microtype}
\usepackage{graphicx}
\usepackage{subfigure}
\usepackage{booktabs} % for professional tables
\usepackage{amsmath}
\usepackage{amssymb}
\usepackage{mathtools}
\usepackage{amsthm}
\usepackage{amsfonts}

\usepackage[backref=page,hidelinks]{hyperref}
\usepackage[capitalize,noabbrev]{cleveref}

\usepackage[toc,page,header]{appendix}
\usepackage{minitoc}
\usepackage{color, colortbl}
\definecolor{LightCyan}{rgb}{0.88,1,1}

\theoremstyle{plain}
\newtheorem{theorem}{Theorem}[section]
\newtheorem{proposition}[theorem]{Proposition}
\newtheorem{lemma}[theorem]{Lemma}
\newtheorem{corollary}[theorem]{Corollary}
\theoremstyle{definition}
\newtheorem{definition}[theorem]{Definition}
\newtheorem{assumption}[theorem]{Assumption}
\theoremstyle{remark}

\usepackage{microtype}
\usepackage{amssymb,amsmath, amsthm}
\usepackage{wrapfig, blindtext}
\usepackage{booktabs} % for professional tables
\usepackage{amssymb}
\usepackage{framed}
\usepackage{mathrsfs}
\usepackage{multirow}
\usepackage{enumerate,pifont,bigdelim,bbding}
\usepackage{multirow} 
\usepackage{bm}
\usepackage{tikz}
\definecolor{gray}{rgb}{0.9, 0.9, 0.9}
\usepackage{enumitem}

\theoremstyle{plain}
\input{mysymbol.sty}
\usepackage[textsize=tiny]{todonotes}

\newcommand{\bignorm}[1]{\left\lVert#1\right\rVert}
\newcommand{\floor}[1]{\lfloor #1 \rfloor}
\newcommand{\ceil}[1]{\lceil #1 \rceil}
\newcommand{\wo}[1]{\widetilde{\mathcal{O}}\left( #1 \right)}
\newcommand{\bo}[1]{\mathcal{O}\left( #1 \right)}
\newcommand{\E}{\mathbb{E}}

\newcounter{relctr} %% <- counter for relations
\everydisplay\expandafter{\the\everydisplay\setcounter{relctr}{0}} %% <- reset every eq
\newcommand\labelrel[2]{%
  \begingroup
    \refstepcounter{relctr}%
    \stackrel{\textnormal{(\alph{relctr})}}{\mathstrut{#1}}%
    \originallabel{#2}%
  \endgroup
}
\AtBeginDocument{\let\originallabel\label}
\icmltitlerunning{\hfill Multi-Level Monte Carlo Actor-Critic (MAC)\hfill\thepage}

\begin{document}
\doparttoc % Tell to minitoc to generate a toc for the parts
\faketableofcontents % Run a fake tableofcontents command for the partocs

%\parttoc % Insert the document TOC

\twocolumn[
% \icmltitle{Tightest Mixing Time Dependence in\\ Average-Reward Actor-Critic via Multi-Level Monte Carlo}
\icmltitle{Beyond Exponentially Fast Mixing in Average-Reward \\Reinforcement Learning via Multi-Level Monte Carlo Actor-Critic}
\icmlsetsymbol{equal}{*}
\begin{icmlauthorlist}
\icmlauthor{Wesley A. Suttle}{equal,arl}
\icmlauthor{Amrit Singh Bedi}{equal,umd}
\icmlauthor{Bhrij Patel}{umd}
\icmlauthor{Brian M. Sadler}{arl}
\icmlauthor{Alec Koppel}{amz}
\icmlauthor{Dinesh Manocha}{umd}
\end{icmlauthorlist}

\icmlaffiliation{umd}{Department of Computer Science, University of Maryland, College Park, USA. }
\icmlaffiliation{arl}{U.S. Army Research Laboratory, Adelphi, MD, USA.}
\icmlaffiliation{amz}{JP Morgan Chase AI Research,  USA. }

\icmlcorrespondingauthor{Wesley A. Suttle}{wesley.a.suttle.ctr@army.mil}
%\icmlcorrespondingauthor{Alec Koppel}{aekoppel314@gmail.com }

\icmlkeywords{Machine Learning, ICML}

\vskip 0.3in
]



\printAffiliationsAndNotice{\icmlEqualContribution} % otherwise use the standard text.

\begin{abstract}
Many existing  reinforcement learning (RL) methods employ stochastic gradient iteration on the back end, whose stability hinges upon a hypothesis that the data-generating process mixes exponentially fast with a rate parameter that appears in the step-size selection. Unfortunately, this assumption is violated for large state spaces or settings with sparse rewards, and the mixing time is unknown, making the step size inoperable. In this work, we propose an RL methodology attuned to the mixing time by employing a multi-level Monte Carlo estimator for the critic, the actor, and the average reward embedded within an actor-critic (AC) algorithm. This method, which we call \textbf{M}ulti-level \textbf{A}ctor-\textbf{C}ritic (MAC), is developed especially for infinite-horizon average-reward settings and neither relies on oracle knowledge of the mixing time in its parameter selection nor assumes its exponential decay; it, therefore, is readily applicable to applications with slower mixing times. Nonetheless, it achieves a convergence rate comparable to the state-of-the-art AC algorithms. We experimentally show that these alleviated restrictions on the technical conditions required for stability translate to superior performance in practice for RL problems with sparse rewards.
\end{abstract}

\section{Introduction} \label{Introduction}

% \begin{itemize}
%     \item Stochastic Optimization Important in ML
%     \item More difficult with Markovian data than iid
%     \item SOTA methods assume knowing mixing time
%     \item Mixing time not known in many cases
%     \item Prior works connect mixing time to TD learning in discounted reward setting
%     \item (Contribution) Getting rid of mixing time dependence in DRL critic network for average reward criterion
%     \item (Contribution) Works on Non-convex problems
%     \item (Contribution) No geometric assumption on mixing time
%     \item (Contribution) Convergence comparable to SOTA

% \end{itemize}
Modern machine learning (ML) techniques have enabled analyzing and making predictions from large-scale data. This is achieved through backpropagation in neural networks \citep{hinton2006fast}, cloud processing of industrial data sets \citep{mcafee2012big}, complex event simulators \citep{silver2016mastering}, and deep feature extraction \cite{krizhevsky2017imagenet}, among other innovations. However, a crucial underlying aspect of these developments is whether training data is sufficiently informative. To put this in quantitative terms, most ML training mechanisms hinge upon training samples being independent and identically distributed (i.i.d.), which is often violated in real-world problems, such as natural language \citep{liu2021federated}, financial markets \citep{heaton2016deep}, and robotics \citep{gu2016deep}, where data exhibits temporal dependence. Reinforcement learning (RL) algorithms, in particular, are limited by this constraint, as the data is inherently Markovian, owing to the fact that RL problem is most commonly represented mathematically as a Markov Decision Process (MDP) \citep{sutton1988}. 
%
% Modern machine learning methods operate by incremental updates on subsets of large-scale data sets in order to obtain learned parameters that yield effective performance once they reach steady-state. For instance, backpropagation of deep neural networks \cite{}, cloud processing of industrial data sets \cite{}, and artificial intelligence for game playing, or the fitting of natural language models all follow this basic recipe \cite{}. A key driver of whether the algorithms on the back end eventuate in effective performance is whether training examples are sufficiently informative. To put this in quantitative terms, most ML training mechanisms hinge upon samples being independent and identically distributed (i.i.d.), which is violated in problems where data exhibits temporal dependence, such as natural language \cite{}, financial markets \cite{}, and robotics \cite{}. Reinforcement learning algorithms, in particular, cannot process i.i.d. samples since data inherihently exhibits Markovian dependence, owing to the fact that it is most commonly  represented mathematically as a Markov Decision Process (MDP). 
%
%\renewcommand{\arraystretch}{1.2}
\begin{table*}[t]
\centering
%\small
%	\vspace{0.05cm}
	\caption{ This table compares the total sample complexity of actor-critic (AC) algorithms available in the literature. To our knowledge, this is the first AC algorithm with an explicit optimal dependence on the underlying mixing time defined as $\tau_{mix} :=\max_{t \in [T]} \tau_{mix}^{\theta_t}$ where $\theta$ is the policy parameter (see Sec. \ref{convergence_analysis}) for detailed discussion).  We also remark that our proposed approach is oblivious to mixing time. }
 \vspace{2mm}
%\begin{threeparttable}
	\resizebox{0.8\textwidth}{!}{\begin{tabular}{|c|c|c|c|c|c|}
		\hline
	\multirow{2}{*}{References} & \multicolumn{2}{c|}{Sampling}  & \multirow{2}{*}{Total complexity}  & \multirow{2}{*}{Reward}  & \multirow{2}{*}{Fast mixing} 
	\\ 
	\cline{2-3}
	 ~ & ~~Actor~~ & Critic & ~  &  ~&
	 \\
	 \hline{}
		\citep{wang2019neural}  & i.i.d.\ & i.i.d.\ & $\mathcal{O}(\epsilon^{-4})$  &Discounted & Required\\ \cline{1-6}
	  \citep{kumar2019sample} & i.i.d.\ & i.i.d.\ & $\mathcal{O}(\epsilon^{-4})$  & Discounted & Required \\  \cline{1-6}
	         \citep{qiu2021} & i.i.d.\ & Markovian & $\tilde {\mathcal{O}}(\epsilon^{-3})$  & Average& Required \\\cline{1-6}
{\citep{xu2020improving}} & {{Markovian}} & {{Markovian}} & {{$\tilde {\mathcal{O}}( \epsilon^{-2})$}} & Average & Required
\\
\cline{1-6}
{\citep{wu2020finite}} & {{Markovian}} & {{Markovian}} & {{$\tilde {\mathcal{O}}( \epsilon^{-2.5})$}} & Average & Required
\\
\cline{1-6}
{\citep{chen2022finite}} & {{Markovian}} & {{Markovian}} & {{$\tilde {\mathcal{O}}( \epsilon^{-2})$}} & Average & Required
\\
\cline{1-6}
  \rowcolor{LightCyan} {{\textbf{This work}}} & {{Markovian}} & {{Markovian}} & {{$\tilde {\mathcal{O}}(\tau_{\text{mix}}^2\cdot \epsilon^{-2})$}} & {Average} & {Not required}%& \textbf{No} 
\\
% 		\\ \cline{1-5}
% 	         \multirow{2}{*}{Natural Actor-Critic} & (Wang et al., 2019) \cite{wang2019neural} & i.i.d.\ & i.i.d.\ & $\mathcal{O}(\epsilon^{-4})$ \\ \cline{2-5}
% 		~ & \cellcolor{blue!15}{\textcolor{red}{This paper}} & \cellcolor{blue!15}{\textcolor{red}{Markovian}} & \cellcolor{blue!15}{\textcolor{red}{Markovian}} & \cellcolor{blue!15}{\textcolor{red}{$\mathcal{O}(\epsilon^{-3})$}} \\ \cline{2-5}
\hline
	\end{tabular}}
%
%
%\vspace{-3mm}
%\end{threeparttable}
\label{table_mixing}
\end{table*}
For this reason, as well as the numerous applications of RL in recent years \citep{li2019reinforcement}, we focus on algorithms for RL methods when data exhibits Markovian dependence. 

Under the Markovian sampling setting, many convergence analyses of iterative methods for RL exist \citep{qiu2021finite,xu2020improving} and typically consider a critical assumption about the rate at which the MDP's transition dynamics converge to stationary distribution for a fixed policy. To establish the analysis, restrictions are placed on \emph{mixing time} ($\tau_{mix}$): (1) prior oracle knowledge of mixing time is employed to determine an optimal step-size selection,  as in \citep{duchi2012, bresler2020}; or (2) mixing time decays exponentially fast, such that the data is asymptotically i.i.d. \citep{qiu2021}. In this work, we are interested in developing RL algorithms with performance certificates without the aforementioned conditions. 

For instance, consider an RL problem where the agent must navigate through a continuous state space, such as a robot reaching a target location or a self-driving car traversing a complex road network. In these cases, the transition dynamics can be highly non-linear with sparse rewards, and the agent may have to explore many states before locating any rewards. In addition, if the environment's dynamics are highly random or there are many obstacles and the agent can get stuck in certain states for a long time, the total variation distance to the steady state decreases slowly, i.e.,  the mixing rate is slow and hence have a large mixing time. These issues often manifest in stationary MDPs that are simply weakly connected by a few distinct regions, which could be defined, e.g., by seasonality in data or distinct learning ``tasks" comprised of similar states and sub-goals as detailed in \citet{riemer2021continual}.
In summary, many RL environments exhibit a slower than exponential mixing rate due to high dimensionality, intrinsic volatility, sparse rewards, or that they contain distinct sub-tasks.

We seek RL methodologies attuned to environments that mix slowly, especially in the context of  actor-critic (AC), due to the fact that it underlies much of modern deep RL \citep{Konda2000}. 
%Its asymptotic convergence in the average-reward setting is well-established via tools from dynamical systems \citep{borkar2000ode,borkar2009stochastic}. 
%More recently, its non-asymptotic performance has received attention \citep{kumar2019sample,qiu2021finite}, motivated by efforts to improve the sample complexity of DRL and the fact that focusing on the average-reward criterion can improve performance at test time \citep{zhang2021}. 
As previously noted, existing results (cf. Table \ref{table_mixing}) hinge upon either i.i.d. \citep{kumar2019sample} or exponentially fast mixing \citep{qiu2021,qiu2021finite}. We, therefore, aim to come up with a variant of actor-critic that does not possess these limitations. To do so, inspired by \citet{dorfman2022}, we develop a multi-level Monte Carlo gradient estimator and adaptive learning rate for the average reward, actor, \emph{and} critic, called Multi-level Monte Carlo Actor-Critic (MAC). We compare the sample complexity of different methods in Table \ref{table_mixing}. 
Our main contributions are:
\begin{itemize}
    \setlength\itemsep{0.25em}
    \item We develop a variant of multi-level Monte Carlo for the average reward, policy gradient, and temporal difference estimates, which together comprise Multi-level Monte Carlo Actor-Critic (MAC) algorithm. 
    \item We establish the convergence rate dependence of the proposed MAC algorithm on the mixing time without any assumption on its decay rate, which is alleviates prior exponentially fast mixing conditions. 
    \item Despite the two-timescale nature of MAC, our use of a modified Adagrad stepsize in the actor allows us to obtain final sample complexity of $\tilde {\mathcal{O}}(\epsilon^{-2})$, instead of the $\tilde {\mathcal{O}}(\epsilon^{-2.5})$ of previous two-timescale analyses.
    %\item Independence of mixing time without an exponential decay assumption while reaching SOTA convergence
    \item We perform initial proof of concept experiments and observe that MAC outperforms vanilla actor-critic for settings with sparse rewards. 
    %
    %\item Novel application of Multi-level Monte Carlo-Adagrad (MAG) from \citet{dorfman2022} to average reward temporal difference with Markovian sampling in critic
\end{itemize}

\subsection{Related Works} \label{Related_Works}
We provide a brief overview of the related works here. Please refer to Appendix \ref{Related_Works_appendix} for a detailed context. 
%Actor-critic \citep{Konda2000} comprises algorithms that alternate between value function estimation (critic) and policy search updates (actor), which may be seen as a form of policy iteration \cite{bertsekas2011approximate} that incorporates stochastic approximation \cite{borkar1997actor}. We discuss each facet separately, before launching into their fusion.

%\textbf{AC}  introduced the AC algorithm. % In contrast to the aforementioned studies of the AC algorithm, our study focuses on the Markovian sampling in the average reward case without any assumptions on mixing time and establishes SOTA AC convergence rate.

\textbf{TD Learning.}  For discounted TD with Markovian samples, \citet{bhandari2018} established finite-time convergence bounds which scale linearly with mixing time $\tau_{mix}$. \citet{dorfman2022} then improved the rate to be proportional to the $\sqrt{\tau_{mix}}$ using a multi-level gradient estimator and adaptive learning rate. \citet{qiu2021} studied TD under the average reward setting, which also imposes exponentially fast mixing that manifests in an additional logarithmic term in the sample complexity. These results all hinge upon imposing restrictive conditions on mixing time.

\textbf{Policy Gradient.} More recently, its sample complexity has been established for a variety of settings: for tabular \cite{bhandari2019global,agarwal2020optimality} and softmax policies \cite{mei2020global}, rates to global optimality exist. For general parameterized policies, early works focused on ``policy improvement" bounds \cite{pirotta2013adaptive,pirotta2015policy}, and more recently, rates towards stationarity \cite{bedi2022hidden} and local extrema \cite{zhang2020global} have been studied, and under special neural architectures, globally optimal solutions \cite{wang2019neural,leahy2022convergence} are achievable. This topic is an active area of work -- we merely identify that these performance certificates all require the mixing rate going to null exponentially fast.

\textbf{Actor-Critic.} As previously mentioned, the stability of actor-critic was initially focused on asymptotics \cite{borkar1997actor}. More recently, its non-asymptotic rate has been derived under i.i.d. assumptions \cite{kumar2019sample,wang2019neural}, and more recently under a variety of different types of Markovian data -- see Table \ref{table_mixing}. However, these results impose that any temporal correlation of data across time vanishes exponentially fast as quantified by the mixing rate. In this way, we are able to match \cite{chen2022finite} but without these restrictions. 

\section{Problem Formulation}
We consider a reinforcement learning problem with an average reward criterion, which can be mathematically defined as a Markov Decision Process (MDP), i.e., a tuple  $\mathcal{M} := (\mathcal{S}, \mathcal{A}, p, r)$. Here, $\mathcal{S}$ is a finite state space; $\mathcal{A}$ is a finite action space; $p(\cdot~|~s,a)$ is a distribution that determines transition to the next state $s^\prime$, and $r: \mathcal{S} \times \mathcal{A} \to [0, r_{\max}]$ is a bounded reward function that informs the merit of selecting action $a$ when starting in state $s$. A policy $\pi (\cdot ~|~ s)$ of an MDP maps the state $s$ to the probability distribution over actions $a$. Formally, $\pi: \mathcal{S} \to \triangle(\mathcal{A})$, where $\triangle(\mathcal{A})$ is the set of probability distributions over $\mathcal{A}$. 
%
%\subsection{Average  Reward Maximization}
%
In the average reward setting, we seek to find a policy $\pi$ such that the long-term average reward is given by $J(\pi):=\lim_{T\rightarrow \infty}\mathbb{E}\left[\frac{1}{T}\sum_{t=0}^Tr(s_t,a_t)\right]$ is maximized. In practice, when the state space is large, it is difficult to search over a general class of policies since its parameterization scales with $|\mathcal{S}|$. Therefore, we restrict focus to the case that $\pi$ is parameterized by a vector $\theta \in \mathbb R^d$, where $d$ denotes the parameter dimension, which leads to the notion of a parameterized policy $\pi_{\theta}$. Optimizing the average reward with respect to policy parameters $\theta$ is the main goal of this work, which we formalize as:
\begin{align}\label{eq:policy_opt}
    \max_{\theta} J(\theta):=\lim_{T\rightarrow \infty}\mathbb{E}_{s_{t+1}\sim p(\cdot|s_t,a_t), a_t\sim\pi_{\theta}(\cdot|s_t)}\left[R_T\right],
\end{align}
where $R_T:=\frac{1}{T}\sum_{t=0}^Tr(s_t,a_t)$. Denote as $d^{\pi_{\theta}}$ the unique stationary state distribution induced by policy $\pi_{\theta}$. Then we can also write $J(\theta)= \mathbb{E}_{s\sim d^{\pi_{\theta}}, a\sim {\pi_{\theta}}}[r(s,a)]$.
% We note that we can ensure the uniqueness of $d^{\pi_{\theta}}$ by assuming that the Markov chain induced by ${\pi_{\theta}}$ is irreducible and aperiodic, which is a standard assumption in the literature \citep[Assumption 1]{qiu2021}. 
%
It turns to be essential to further algorithm development to define the action-value ($Q$) function as 
\begin{align}\label{q_function}
    Q^{\pi_{\theta}}(s,a) =& \mathbb{E}\Bigg[\sum_{t=0}^{\infty}\mathbb [r(s_t,a_t) - J(\theta)]\Bigg],
    \end{align}
    such that $s_0 = s, a_0 = a$, and action $a\sim \pi_{\theta}$. This implies that we can write the state value function as
    \begin{align}\label{v_function}
     V^{\pi_{\theta}}(s) =& \mathbb E_{a \sim \pi_{\theta}(\cdot|s)}[Q^{\pi_{\theta}}(s,a)].
\end{align}
From \eqref{q_function} and \eqref{v_function}, we can write the value of a state $s$, in terms of another via Bellman's Equation as \cite{puterman2014markov}
\begin{align}\label{bellman}
V^{\pi_{\theta}}(s)
&= \mathbb E[r(s,a) - J(\theta) + V^{\pi_{\theta}}(s')],
\end{align}
where the expectation is over $a \sim \pi_{\theta}(\cdot|s), s'\sim p(\cdot |a, s)$. Next, we shift to defining the standard actor-critic framework to solve \eqref{eq:policy_opt}, {in order to illuminate its merits and drawbacks.}
\subsection{Decay Rates of Mixing Times}
It is inherent to RL that the data-generating mechanism is state-dependent and Markovian, which means that assumptions that trajectory data is independent and identically distributed do not hold \citep{wang2019neural,kumar2019sample,qiu2021finite}. {That is, the noise driving the estimation error of the algorithm updates is heteroscedastic (variance is heterogeneous).} Because of this challenge, various technical conditions have been considered to quantify the degree of correlation in data across time, mostly inherited from the applied probability literature -- see \cite{levin2017markov}. Most prior stability and sample complexity results of RL algorithms for the average reward setting are defined in terms of the \emph{mixing time}, which is the minimum time at which the transition dynamics are near the long-term steady-state distribution induced by a policy $\pi_\theta$, as formalized next.
%\textbf{Limitations}: To discuss the limitations of the vanilla actor critic approach, let us define the concept of mixing time explicitly here. 
%
%In this section, we derive the definition of mixing time. To do so, it is necessary to first establish the definition for the \emph{total variation} between two distributions, $P$ and $Q$, with event space $F$. It is the largest difference between the two distributions given the same event. Formally, it is the following:
%\begin{definition}[Total Variation]
%\begin{equation} \label{TV}
%D_{TV}(P,Q) := \sup_{A\in F} \lvert P(A) - Q(A)\rvert
%\end{equation}
%\end{definition}
%
\begin{definition}[$\epsilon$-Mixing Time] Let $d^{\pi_{\theta}}$ denote the stationary distribution of the Markov chain induced by $\pi_{\theta}$. Define $P_{\theta}(s'|s) = \int_{\mathcal{A}} p(s'|s,a) \pi_{\theta}(a|s) da$. The $\epsilon$-mixing time of the Markov chain induced by $\pi_{\theta}$ is defined as
\begin{equation}\label{mixing time}
\tau_{mix}^\theta(\epsilon) := \inf \{t: \sup_{s\in \mathcal{S}} \|P^t_{\theta}(\cdot | s)- d^{\pi_{\theta}}(\cdot) \|_{TV} \leq \epsilon\},
\end{equation}
where $\|\cdot\|_{TV}$ is the total variation distance. The conventional mixing time is defined as $\tau_{mix}^\theta:=\tau_{mix}^\theta(1/4)$.
\end{definition}

	\textbf{Limitations.} In all of the earlier works mentioned in Table \ref{table_mixing}, a crucial and common assumption is regarding the exponentially fast decay rate  of the mixing time. Specifically, all the works assume that there exist $\zeta>0$ and $\rho\in(0,1)$ such that, for all $\theta$, it holds that
	$\sup_{s\in \mathcal{S}} \|P^t_{\theta}(\cdot | s)- d^{\pi_\theta}\|_{TV} \leq \zeta \rho^t$. This stipulates that exponentially fast mixing must hold \textit{uniformly} for all induced Markov chains. Also, to proceed with the convergence analysis in the works  mentioned in Table \ref{table_mixing}, knowledge of $\zeta$ and $\rho$ is required for the optimal step size selection, which is usually unknown in practice. Moreover, there is a wide range of applications where polynomial decay rates have some fundamental role to play in defining RL algorithms that can generalize well across tasks - see \citep{riemer2021continual} for a detailed description.
	
	Therefore, in this work, we are interested in going beyond the exponentially mixing requirements and seek to develop actor-critic algorithms which do not require access to mixing time values a priori for optimal performance. We present our proposed algorithm in the next section.
	%	%  citing
%	\textcolor{red}{here we do a review of:\\
%		(i) various assumptions on the mixing time\\ 
%		(ii) a brief survey of existing rate analyses results for various flavors of PG, actor-critic, etc. \\
%		(iii) how our technical focus generalizes these results because we dont impose any condition the decay rate of the mixing time. We include some commentary about how polynomial decay rates have some fundamental role to play in defining RL algorithms that can generalize well across tasks, citing }
%	
	
\section{Actor-Critic Method}
\subsection{Elements of Actor-Critic}\label{sec:actor_critic}
We start by providing a quick recap of the standard actor-critic (AC) algorithm in average reward RL settings. The AC algorithm operates by alternating updates between the actor and critic, which are respectively defined in terms of gradient updates to policy parameters $\theta$ and estimates of the value function $V^{\pi_{\theta}}(s)$ based on the fixed point recursion implied by Bellman's equation \eqref{bellman}. To do so, we proceed by writing down a gradient ascent iteration for the maximization in \eqref{eq:policy_opt} given by
\begin{align}
    \theta_{t+1}= \theta_t + \alpha_t \nabla_{\theta} J(\theta_t),
\end{align}
where $\alpha_t$ is the step size. From the Policy Gradient (PG) Theorem \cite{williams1992simple,sutton1999policy}, it is well-known that $\nabla_{\theta} J(\theta_t)$ takes the explicit form:
\begin{align}\label{policy_gradient}
    \nabla_{\theta} J(\theta) = \mathbb{E}_{(s,a,s')\sim\Gamma_\theta} \left[\delta^{\pi_{\theta}}\cdot \nabla_\theta \log \pi_{\theta} (a|s)\right],
\end{align}
with the \emph{temporal difference} (TD) $\delta^{\pi_{\theta}}$ defined as  \citep{sutton1988}:
\begin{align}
    \delta^{\pi_{\theta}}:=r(s,a)-J(\theta)+ V^{\pi_\theta}(s')-V^{\pi_\theta}(s),
\end{align}
and  $\Gamma_\theta:=s\sim d^{\pi_{\theta}}, a\sim {\pi_{\theta}}, s'\sim p(\cdot|s,a)$ is the short notation for the joint distribution. 
We note that there are two parts in the expression of PG in \eqref{policy_gradient}: $\nabla_\theta \log \pi_{\theta} (a|s)$, the score function which comes from the policy parameterization. The TD term $\delta^{\pi_{\theta}}$ is defined in terms of rearranging the $V^{\pi_\theta}(s)$ term in \eqref{bellman} to the other side of the expression, and group expectations. Observe that the differential value function $V^{\pi_\theta}(s')-V^{\pi_\theta}(s)$  distinguishes the PG \eqref{policy_gradient} in the average-reward case different from the discounted setting. 

\textbf{Critic update:} We restrict focus to the case where the value function $V^{\pi_\theta}(s)$ is estimated by the inner product between a given feature map $\phi(s)$ and a weight vector $\omega$, which can be shown to be exact under some special cases such as linear MDP where the assumption of realizability is  met \citep{roy1997, bhandari2018, dorfman2022, qiu2021}. Hence, we can write $V_{\omega}(s) = \langle\phi(s),\omega \rangle$ where
%which implies that we can write \textcolor{red}{add a comment or footnote about linear MDPs, realizability, Bellman eluder dimension, etc. When does linear basis expansion converge to the right/wrong value in the discounted and average reward setting?} 
%
%\begin{align}
%    \; .
%\end{align}
%
 $V_{\omega}(s)$ denotes the estimator to $V^{\pi_\theta}(s)$ in terms of parameters $\omega \in \mathbb R^m$ and feature map $\phi: \mathcal{S}\to \mathbb R^m$ of state $s$ to $m$-dimensional space such that $\lVert \phi(s) \rVert \leq 1$ for all $s \in \mathcal{S}$. TD learning-style updates are then used to find $\omega$, which minimizes the error $G(\omega)$ defined as
\begin{align}\label{value_function_problem}
    \min_{\omega\in\Omega} G(\omega):= \sum_{s \in S} d^{\pi_{\theta}} (V^{\pi_{\theta}}(s) - V_{\omega}(s))^2.
\end{align}
%
%where we define $C$ be a diagonal matrix with entries correspond to $d^{\pi_{\theta}}$. 
The TD(0) update for the critic parameter $\omega$ is given as 
\begin{align}\label{Critic_update_00}
    \omega_{t+1}=& \Pi_{\Omega}\big[\omega_t-\beta_t \big(r(s_t,a_t)-J(\theta_t)+ \langle\phi(s_{t+1}),\omega_t\rangle
    \nonumber
    \\
    &\hspace{25mm}-\langle\phi(s_t),\omega_t\rangle\big)\phi(s_t)\big],
\end{align}
where $\beta_t$ is the critic learning rate. We remark that the critic update in \eqref{Critic_update_0} requires knowledge of $J(\theta_t)$ (time-averaged reward), which is typically not available. We can replace this unknown quantity with a recursive estimate for the average reward given by $\eta_{t+1}= \eta_t-\gamma_t (\eta_t-r(s_t,a_t))$. Putting this all together, we can write the vanilla actor-critic scheme as
\begin{align}\label{Critic_update_0}
	    \eta_{t+1}=& \eta_t-\gamma_t \cdot f_t && \text{(reward tracking)}
	    \nonumber
	    \\
    \omega_{t+1}=& \Pi_{\Omega}\big[\omega_t-\beta_t\cdot g_t\big],&& \text{(critic update)}
    \nonumber\\
        \theta_{t+1}=& \theta_t + \eta_t \cdot  \delta^{\pi_{\theta_t}}\cdot h_t,&& \text{(actor update)}
\end{align}
where we have 
\begin{align}\label{stochastic_gradients}
 f_t=& \eta_t-r(s_t,a_t),
 \nonumber
 \\
    g_t=&\big(r(s_t,a_t)-\eta_t+ \langle\phi(s_{t+1})-\phi(s_{t}),\omega_t\rangle\big)\phi(s_t),
    \nonumber
    \\
    h_t=& \delta^{\pi_{\theta_t}}\cdot \nabla_\theta \log \pi_{\theta_t} (a_t|s_t),
    \nonumber
    \\
    \delta^{\pi_{\theta_t}}=& r(s_t,a_t)-\eta_t+ \langle\phi(s_{t+1})-\phi(s_{t}),\omega_t\rangle.
\end{align}

As previously mentioned, the stability of \eqref{Critic_update_0}-\eqref{stochastic_gradients} can only be ensured under the exponentially fast mixing condition, which can preclude sparse-reward or large state space cases. For this reason, we develop an augmentation of actor-critic that alleviates this restriction in the following subsection. 
% \red{Discuss about the limitations here.}

\begin{algorithm}[t]
	\caption{\textbf{M}ulti-level Monte Carlo \textbf{A}ctor-\textbf{C}ritic (MAC)}
	\label{alg:PG_MAG}
	\begin{algorithmic}[1]
		\STATE \textbf{Initialize:} Policy parameter $\theta_0$, actor step size $\alpha_t$, critic step size $\beta_t$, average reward tracking step size $\gamma_t$, initial state $s_1^{(0)} \sim \mu_{0}(\cdot)$, maximum rollout length $T_{\max}$.
		%\REPEAT
		%$\STATE Initialize $noChange = true$.
		\FOR{$t=0$ {\bfseries to} $T-1$}
		%\IF{$x_i > x_{i+1}$}
		\STATE Sample level length $j_t \sim \text{Geom}(1/2)$
		\FOR{$i = 1, \dots, 2^{j_t}$}
		\STATE Take action $a^i_t \sim \pi_{\theta_t}(\cdot | s^i_{t})$
		\STATE Collect next state $s^{i+1}_{t} \sim P(\cdot | s_t^{i}, a^i_{t})$
		\STATE Receive reward $r_t^{i} = r(s^i_t, a^i_t) $
		%\STATE $\delta_{i}^{(t)} = r_i^{(t)} - \nu^{(t)} + V^{(t)}(s_{i+1}^{(t)}) - V^{(t)}(s_{i}^{(t)})$
		%\STATE $\widehat{\nabla J_i^{(t)}} = \delta_i^{(t)}\nabla\log \pi_{\theta_t}(a_i^{(t)}|s_i^{(t)})$
		
		%\STATE Swap $x_i$ and $x_{i+1}$
		%\STATE $noChange = false$
		%\ENDIF
		\ENDFOR
		\STATE Compute MLMC gradients $ f_t^{MLMC}$, $ g_t^{MLMC}$, $ 
 h_t^{MLMC}$ via \eqref{MLMC_Gradient}-\eqref{mlmc_4}
		\STATE Update parameters as
		% %
		\begin{align*}\label{udpates_algorotihm}
			\eta_{t+1}=& \eta_t-\gamma_t \cdot  f_t^{MLMC} && \text{(reward tracking)}
			\nonumber
			\\
			\omega_{t+1}=& \Pi_{\Omega}\big[\omega_t-\beta_t\cdot   g_t^{MLMC}\big],&& \text{(critic update)}
			\nonumber\\
			\theta_{t+1}=& \theta_t + \eta_t \cdot  \delta^{\pi_{\theta_t}}\cdot   h_t^{MLMC},&& \text{(actor update)}
		\end{align*}
		%
%		
%		\STATE $g_t^j \coloneqq \frac{1}{2^j}\sum_{i=1}^{2^j} \widehat{\nabla J_i^{(t)}}$
%		\IF{$2^J_t \leq T_{\max}$}
%		\STATE $g_t = g_t^{\text{MLMC}} = g_t^0 + 2^{J_t}(g_t^{J_t} - g_t^{J_t-1})$
%		\ELSE
%		\STATE $g_t = g_t^{\text{MLMC}} = g_t^0$
%		\ENDIF
%		\STATE $\alpha_t = \alpha/\sqrt{\sum_{k=1}^t \lVert g_k\rVert^2}$
%		\STATE $\theta_{t+1} = \theta_t + \alpha g_t$
%		\STATE $s_1^{(t+1)} = s_{2^{J_t}}^{(t)}$
		\ENDFOR
		%\UNTIL{$noChange$ is $true$}
	\end{algorithmic}
\end{algorithm}

\subsection{Multi-level Monte Carlo Actor-Critic}\label{sec:mac}
Recent work of \citet{dorfman2022} has developed the use of Multi-level Monte Carlo techniques together with AdaGrad step-size selection to develop a gradient estimator for Markovian data in stochastic optimization settings. We build upon these techniques in putting forth an MLMC gradient estimator for the actor, critic, and reward tracking. In doing so, we also allow for the sampling distribution for the critic to be Markovian. Specifically, we propose to replace the stochastic gradients $f_t$, $g_t$, and $h_t$ in  \eqref{Critic_update_0} with the following MLMC gradients. Letting $J_t \sim \text{Geom(}{1}/{2}\text{)}$ and fixing a maximum number $T_{max}$ of samples, we collect a trajectory $\mathcal{T}_t:=\{s_t^i,a_t^i,r_t^i,s_t^{i+1}\}_{i=1}^{2^{J_t}}$ for each $t$ by interacting with the environment using policy parameter vector $\theta_t$. For the policy gradient estimate, for example, we then construct the MLMC estimate 
\begin{equation}\label{MLMC_Gradient}
h_t^{MLMC} = h_t^0 +
\begin{cases}
2^{J_t}(h_t^{J_t} - h_t^{J_t - 1}),& \text{if } 2^{J_t}\leq T_{\max}\\
0,              & \text{otherwise}
\end{cases}             
\end{equation}
with $h_t^j = \frac{1}{2^j}\sum_{i=1}^{2^j}h (\theta_t;s_t^i,a_t^i)
$
aggregating $2^j$ gradients:
\begin{align}\label{mlmc_2}
  h(\theta_t;s_t^i,a_t^i) &= \delta^{\pi_{\theta_t}}_i\cdot \nabla_\theta \log \pi_{\theta_t} (a_t^i|s_t^i),
    \\
    \delta^{\pi_{\theta_t}}_i &= r(s_t^i,a_t^i)-\eta_t+ \langle\phi(s_{t+1}^i)-\phi(s_{t}^i),\omega_t\rangle.    \nonumber
\end{align}
We can formulate estimates analogous to \eqref{MLMC_Gradient} for the reward tracking gradient $f_t^{MLMC}$ and critic gradient $g_t^{MLMC}$ by using corresponding versions of \eqref{mlmc_2}:
% Similar to the update in \eqref{MLMC_Gradient} and \eqref{mlmc_2}, we can write MLMC updates for reward tracking gradient estimator $f_t^{MLMC}$ and critic update $g_t^{MLMC}$:
%
\begin{align}
    f(\eta_t; s_t^i, a_t^i) &= r(s_t^i, a_t^i) - \eta_t, \label{mlmc_3} \\
    g(\omega_t; s_t^i, a_t^i) &= \delta^{\pi_{\theta_t}}_i \cdot \phi(s_t^i). \label{mlmc_4}
\end{align}
The multi-level gradient in \eqref{MLMC_Gradient} is different from the one in \eqref{Critic_update_0} where we only need one sample $(s_t,a_t,s_{t+1})$ to evaluate the actor and critic updates.
%
%But we allow Markovian sampling and hence, we will collect a Monte Carlo rollout trajectory of length $J_t$ at each $t$.  Hence, we collect $\mathcal{T}_t:=\{s_t^i,a_t^i,s_t^{i+1}\}_{i=1}^{J_t}$ for each $t$ by interacting with the environment using policy parameter vector $\theta_t$. Therefore, we can write the stochastic gradient for each $i$ in \eqref{mlmc_per_Step} as 
% %
% \begin{align}
%     g_i (\omega_t)=\big(r(s_t^i,a_t^i)-\mu_t+ \phi(s_{t}^{i+1})^\top\omega_t-\phi(s_t^i)^\top\omega_t\big)\phi(s_t^i).
% \end{align}
% %
%\textcolor{red}{it's unclear if the previous expression is for the actor update or the critic. We need to better clarify which one is which, and write out explicitly the update for the actor and the critic. I am sure $\hat{g}_t$ in \eqref{Critic_update_1} is wrong. It can't be the same update direction for all three iterations, right ? Also, there are a bunch of broken equation references in Algorithm \ref{alg:PG_MAG} What is the definition of $\hat{g}_t$, $\hat{f}_t$, $\hat{h}_t$ }
Overall, the proposed multi-level Monte Carlo actor-critic (MAC) takes the form
%In a similar way, we can also define a stochastic multi-level gradient for actor as well and rewrite the updated proposed actor-critic updates as 
%
%
\begin{align}\label{Critic_update_MAC}
	    \eta_{t+1}=& \eta_t-\gamma_t \cdot  f_t^{MLMC} && \text{(reward tracking)}
	    \nonumber
	    \\
    \omega_{t+1}=& \Pi_{\Omega}\big[\omega_t-\beta_t\cdot  g_t^{MLMC}\big],&& \text{(critic update)}
    \nonumber\\
        \theta_{t+1}=& \theta_t + \eta_t \cdot  \delta^{\pi_{\theta_t}}\cdot h_t^{MLMC},&& \text{(actor update)}
\end{align}
We summarize the proposed algorithm in Algorithm \ref{alg:PG_MAG}.

\section{Non-asymptotic Convergence Analysis}\label{convergence_analysis}

In this section we provide convergence rate and sample complexity results for Algorithm \ref{alg:PG_MAG}. We extend the MLMC analysis of \citet{dorfman2022} to the actor-critic setting, where we combine it with the two-timescale finite-time analysis of \citet{wu2020finite} to obtain non-asymptotic convergence guarantees for MAC (cf. Algorithm \ref{alg:PG_MAG}). Salient features of our approach: \textbf{(1)} it avoids uniform ergodicity assumptions required in previous finite-time analyses \cite{zou2019finite, wu2020finite, chen2022finite}; \textbf{(2)} it explicitly characterizes convergence rate dependence on the mixing times encountered during training; \textbf{(3)} it (i) clarifies the trade-offs between mixing times and MLMC rollout length $T_{\max}$, and (ii) extends the standard analysis to handle additional sources of bias in the MLMC estimator, both of which were missing from the analysis of \citet{dorfman2022}; \textbf{(4)} it leverages modified Adagrad stepsizes to avoid the slower convergence rates of previous two-timescale analyses \cite{wu2020finite} (cf. Theorem \ref{thm:convergence_rate}).
%
% \begin{itemize}
%     \item avoids uniform ergodicity assumptions required in previous finite-time analyses \cite{zou2019finite, wu2020finite, chen2022finite};
%     %
%     \item explicitly characterizes convergence rate dependence on the mixing times encountered during training;
%     %
%     \item (i) clarifies the trade-offs between mixing times and MLMC rollout length $T_{\max}$, and (ii) extends the standard analysis to handle additional sources of bias in the MLMC estimator, both of which were missing from the analysis of \citet{dorfman2022};
%     %
%     \item leverages modified Adagrad stepsizes to avoid the slower convergence rates of previous two-timescale analyses \cite{wu2020finite} (cf. Theorem \ref{thm:convergence_rate}).
% \end{itemize} 

The rest of this section is structured as follows. We first outline standard assumptions (cf. Sec. \ref{aaaumptions}) from the literature and provide some preliminary results. Second, we analyze the policy gradient norm (cf. Sec. \ref{actor_convergence}) associated with Algorithm \ref{alg:PG_MAG}, which provides a preliminary convergence rate and characterizes its dependence on the error arising from the critic estimation procedure, the MLMC bias resulting from the choice of $T_{\max}$ and mixing times encountered, and the bias inherent in using function approximation for the critic. Third, we analyze the convergence (cf. Sec. \ref{criti_convergence}) of the critic estimation error, characterizing its dependence on the MLMC bias and its convergence rate. Finally, we combine the actor and critic analyses to provide our main convergence rate and sample complexity (cf. Theorem \ref{thm:convergence_rate}) results for MAC. To keep the exposition clear, we provide simplified versions of our main results and omit proofs in this section. Mathematically precise statements and detailed proofs of all results are presented in the appendix.
%All the detailed proofs and more detailed statements of the results presented in this section are provided in the appendix. }

\subsection{Assumptions and Propositions}\label{aaaumptions}
The algorithmic setting considered in this paper is that of actor-critic with linear function approximation, where the critic updates correspond to using TD(0) \cite{sutton2018reinforcement} to estimate the state value function. Specifically, we assume that, for a given critic parameter $\omega \in \mathbb{R}^k$ and state $s$, our critic approximator is of the form $V_{\omega}(s) = \phi(s)^T \omega$ [cf. \eqref{value_function_problem}, where $\phi : \mathcal{S} \rightarrow \mathbb{R}^k$ is a given feature mapping that we assume satisfies $\sup_s \norm{ \phi(s) } \leq 1$.

As discussed in Ch. 9 of \cite{sutton2018reinforcement}, for a fixed policy parameter $\theta$, TD(0) with linear function approximation will converge to the minimum of the mean squared projected Bellman error (MSPBE), which satisfies
\begin{equation}\label{critic_colution}
    A_{\theta} \omega = b_{\theta},
\end{equation}
\begin{equation}
    A_{\theta} = \E_{s \sim \mu_{\theta}, a \sim \pi_{\theta}, s' \sim p(\cdot | s, a)} \left[ \phi(s) (\phi(s) - \phi(s'))^T \right], \nonumber
\end{equation}
\begin{equation}
    b_{\theta} = \E_{s \sim \mu_{\theta}, a \sim \pi_{\theta}} \left[ ( r(s, a) - J(\theta) ) \phi(s) \right]. \nonumber
\end{equation}
In what follows, we will use $\omega^*(\theta)$ to denote the fixed point satisfying Eq. \eqref{critic_colution} for a given $\theta$. We will also use $\omega^*_t = \omega^*(\theta_t)$ to denote the fixed point associated with policy parameter vector $\theta_t$ at time $t$.
For a given feature mapping $\phi$, we define the worst-case approximation error to be
\begin{equation}
    \mathcal{E}_{app} = \sup_{\theta} \sqrt{ \E_{s \sim \mu_{\theta}} \left[ \phi(s)^T \omega^*(\theta) - V^{\pi_{\theta}}(s) \right]^2},
\end{equation}
which we assume to be finite. Intuitively, $\mathcal{E}_{app}$ quantifies the quality of the feature mapping: when the features are well-designed, $\mathcal{E}_{app}$ will be small or even 0, while poorly designed features will tend to have higher worst-case error.

Analyses of TD learning typically assume positive definiteness of the matrices $A_{\theta}$ to ensure the solvability of the MSPBE minimization problem and uniqueness of its solutions \cite{bhandari2018, zou2019finite, qiu2021finite}, which we subsequently impose via Assumption \ref{assum:critic_positive_definite}.
\begin{assumption} \label{assum:critic_positive_definite}
    There exist $\lambda > 0$ such that, for all $\theta$, the matrix $A_{\theta}$ is positive definite, its eigenvalues are all bounded and have norm greater than or equal to $\lambda$.
\end{assumption}

As indicated in our description of the algorithm in the previous section, we execute a projection onto a norm-ball with radius $R_{\omega} > 0$, denoted by set $\Omega$, in our critic update step [cf. \eqref{Critic_update_MAC}]. As mentioned in \cite{wu2020finite}, given Assumption \ref{assum:critic_positive_definite}, we can simply take $R_{\omega} = 2 R / \lambda$, since $\norm{b_{\theta}} \leq 2R$ by the boundedness of rewards, and $\norm{A_{\theta}^{-1}} \leq 1 / \lambda$.

In order to establish an ascent-type condition on the policy gradient, we require some regularity conditions which have been considered in recent analyses of model-free RL methods \cite{papini2018stochastic, kumar2019sample, zhang2020global, xu2020improved}, as detailed next.
\begin{assumption} \label{assum:policy_conditions}
    Let $\{ \pi_{\theta} \}_{\theta \in \mathbb{R}^d}$ denote our parameterized policy class. There exist $B, K, L > 0$ such that
    \begin{enumerate}
        \setlength\itemsep{0em}
        \item $\norm{ \nabla \log \pi_{\theta}(a | s) } \leq B$, for all $\theta \in \mathbb{R}^d$,
        \item $\norm{ \nabla \log \pi_{\theta}(a | s) - \nabla \log \pi_{\theta'}(a | s) } \leq K \norm{ \theta - \theta' }$, for all $\theta, \theta' \in \mathbb{R}^d$,
        \item $| \pi_{\theta}(a | s) - \pi_{\theta'}(a | s) | \leq L \norm{\theta - \theta'}$, for all $\theta, \theta' \in \mathbb{R}^d$.
    \end{enumerate}
\end{assumption}

Finally, for our last major assumption we impose a condition on the ergodicity coefficients of the family of state transition kernels $\{ P_{\theta} \}$ induced by the policy class $\{ \pi_{\theta} \}$, where $P_{\theta}(s' | s) = \int_{\mathcal{A}} \pi_{\theta}(a | s) p(s' | s, a) da$. For a fixed transition kernel $P$, defined its ergodicity coefficient to be $\kappa(P) := \sup_{s, s'} \norm{ P(\cdot | s)- P(\cdot | s') }_{TV}$ \cite{mitrophanov2005sensitivity}. Furthermore, for a given $k \in \mathbb{N}$ and fixed $P$, let $P^k$ denote the induced $k$-step transition kernel.
\begin{assumption} \label{assum:ergodicity_coeff}
    For every $\theta$, there exists $k \in \mathbb{N}$ such that the ergodicity coefficient $\kappa(P_{\theta}^k)$ satisfies $\kappa(P_{\theta}^k) < 1$.
\end{assumption}
%
%\textcolor{red}{need to say something about relationship to polynomial decay rates and stuff about continual RL, etc.}
%\red{what is definition of ergodicity coefficient? Is this somewhere in appendix?}
In prior works, related quantities are assumed to go to null exponentially fast (uniform ergodicity) in finite-time analyses of average-reward actor-critic \cite{wu2020finite, qiu2021finite, chen2022finite} and related RL methods \cite{melo2008analysis, bhandari2018, zou2019finite} (Theorem 3.1 of \cite{mitrophanov2005sensitivity} establishes a correspondence). In our case, we merely require it to be upper-bounded by a constant, meaning that the degree of non-stationarity of the transition dynamics cannot be arbitrarily large, and at worst has bounded drift with time. This allows us to better accomodate large state spaces comprised of distinct regions, which may be defined by seasonality. %\textcolor{red}{TODO: explain how it might be used to accommodate settings described in intro?}

We are now ready to provide two important propositions that will be important in the core analysis to follow.
\begin{proposition} \label{lemma:critic_Lipschitz}
    Under Assumptions \ref{assum:critic_positive_definite}-\ref{assum:ergodicity_coeff}, there exists $L_{\omega} > 0$ s.t. $\norm{ \omega^*(\theta) - \omega^*(\theta') } \leq L_{\omega} \norm{ \theta - \theta' }$, for all $\theta, \theta'$.
\end{proposition}
Please refer Lemma \ref{lemma:critic_Lipschitz_app} in the appendix for the proof of Proposition \ref{lemma:critic_Lipschitz}. The next proposition is a generalization of Lemma 3.1 from \citet{dorfman2022}, adapted to our actor-critic setting, that explicitly characterizes the computational cost associated with MLMC rollout length $T_{\max}$. 

Before stating our main results, we first establish a result characterizing the mean and variance of the MLMC gradient estimators $f_t^{MLMC}, g_t^{MLMC}, h_t^{MLMC}$ used in the MAC updates defined in \eqref{Critic_update_MAC}. Since the core result is the same for all three estimators, we formulate and derive the result for a general MLMC estimator $l_t^{MLMC}$. We note that $l_t^{MLMC}$ can be replaced by any one of $f_t^{MLMC}, g_t^{MLMC}, h_t^{MLMC}$ and the result will hold. To prepare to state the result, let a policy parameter $\theta_t$ be given and sample $J_t \sim \text{Geom}(1/2)$. Fix $T_{\max} \in \mathbb{N}$ such that $T_{\max} \geq \tau_{mix}^{\theta_t}$. Fix a trajectory $z_t = \{ z_t^i = (s_t^i, a_t^i, r_t^i, s_t^{i+1}) \}_{i \in [2^{J_t}]}$ generated by following policy $\pi_{\theta_t}$ starting from $s_t^0 \sim \mu_0(\cdot)$. Let $\nabla L(x):=\mathbb{E}_{z \sim \mu_{\theta_t}, \pi_{\theta_t}} \left[ l(x, z) \right]$ be a gradient that we wish to estimate over $z_t$ where $x \in \mathcal{K} \subset \mathbb{R}^k$ is the parameter of the estimator $l$, e.g., $x$ could be $x_t = \theta_t, \eta_t$, or $\omega_t$.
%\textcolor{red}{x here is confusing. I suggest to unify the notation completely with the PG setting and the MLMC estimators defined in Section 3. Not totally clear on the interpretation of \eqref{l_MLMC_Gradient} in the context of Algorithm \ref{alg:PG_MAG}}
%
hence, the MLMC estimator (cf. \eqref{MLMC_Gradient}) becomes
\begin{equation}\label{l_MLMC_Gradient}
l_t^{MLMC} = l_t^0 +
\begin{cases}
2^{J_t}(l_t^{J_t} - l_t^{J_t - 1}),& \text{if } 2^{J_t}\geq T_{\max}, \\
0,              & \text{otherwise.}
\end{cases}             
\end{equation}
We are ready to present our result for the MLMC estimator in Proposition \ref{prop:31}.
\begin{proposition}\label{prop:31}
Let $j_{\max} = \floor{\log T_{\max}}$. Fix $x_t$ measurable w.r.t. $\mathcal{F}_{t-1}$. Assume $\norm{\nabla L(x)} \leq G_L$, for all $x \in \mathcal{K}$, and $\norm{ l_t^N } \leq G_L$, for all $N \in \left[ T_{\max} \right]$. Then
\begin{align}\label{eqn:prop31_mean_main_body}
    \mathbb{E}_{t-1} \left[ l_t^{MLMC} \right] &= \mathbb{E}_{t-1} \left[ l_t^{j_{\max}} \right],  \\
    \mathbb{E} \left[ \norm{ l_t^{MLMC} }^2 \right] &\leq \wo{ G_L^2 \tau_{mix}^{\theta_t} \log T_{\max} }. \label{eqn:prop31_second_moment_main_body}
\end{align}
\end{proposition}
We provide the proof of Proposition \ref{prop:31} with a detailed description of the statement in Lemma \ref{lemma:31_app} in the appendix.

\textbf{Remark.} We note that the corresponding result in \citet{dorfman2022} hides the logarithmic dependence of the second moment bound \eqref{eqn:prop31_second_moment_main_body} on the MLMC rollout length $T_{\max}$, subsuming it into the $\wo{\cdot}$ order notation. When $T_{\max}$ is allowed to grow with time, e.g., by setting $T_{\max} = T$ as in \citet{dorfman2022}, the true impact of using MLMC is not accurately accounted for. Furthermore, a finite value for $T_{\max}$ must be used in practice, so it is important to understand its true effect. We rigorously characterize its effect with Proposition \ref{prop:31}.

In addition, Proposition \ref{prop:31}, its precursor results (see Lemmas \ref{lemma:a5}, \ref{lemma:a6} in appendix), and our extensions of it (see Lemmas \ref{lemma:a6ss}, \ref{lemma:a6sss}, \ref{lemma:31ss}, \ref{lemma:31sss} in appendix) are the critical tools that allow us to smoothly accommodate Markovian sampling and reveal the dependence on mixing times encountered in the analysis. Equation \eqref{eqn:prop31_mean_main_body} is used at many points in the analysis to tie the behavior of our MLMC estimates to that of the lower-bias estimators $f_t^{j_{\max}}, g_t^{j_{\max}}, h_t^{j_{\max}}$, while equation \eqref{eqn:prop31_second_moment_main_body} renders the dependence on $\log T_{\max}$ and mixing time explicitly, and allows us to avoid uniform ergodicity assumptions. These innovations allow us to derive the improved actor and critic convergence analyses presented next.

\subsection{Convergence of the Actor}\label{actor_convergence}

In this section, we take the first step towards establishing convergence of Algorithm \ref{alg:PG_MAG} by providing a bound on the average policy gradient norm. This result explicitly characterizes the actor convergence in terms of its dependence on the average reward tracking and critic estimation error, mixing times encountered during training, MLMC rollout length $T_{\max}$, and the function approximation bias $\mathcal{E}_{app}$. We present our first main result in Theorem \ref{thm:actor_bd}.

\begin{theorem}\label{thm:actor_bd}
Assume $J(\theta)$ is $L$-smooth, $\sup_{\theta} | J(\theta) | \leq M$, and $\norm{ \nabla J(\theta) }, \norm{ h_t^{MLMC} } \leq G_H$, for all $\theta, t$. Let $\alpha_t = \alpha_t' / \sqrt{ \sum_{t=1}^T \norm{h_t^{MLMC}}^2 }$, where $\{ \alpha_t' \}$ is an auxiliary stepsize sequence with $\alpha_t' \leq 1$, for all $t \geq 1$. Then
\small
\begin{align}\label{actor_adagrad}
    \frac{1}{T} \sum_{t=1}^T \mathbb{E} & \left[ \bignorm{ \nabla J(\theta_t) }^2 \right] \leq \bo{\frac{1}{\sqrt{T}} } + \bo{ \frac{1}{T} \sum_{t=1}^T \mathcal{E}(t) } \nonumber\\
    &+ \wo{ \max_{t \in [T]} \tau_{mix}^{\theta_t} \frac{\log T_{\max}}{T_{\max}}} + \bo{ \mathcal{E}_{app} },
\end{align}
\normalsize
where
%
%\begin{align}
   % \mathcal{E}_1(t) &= \mathbb{E} \left[ \norm{ \eta_t - \eta_t^* } \right] + \mathbb{E} \left[ \norm{\omega_t - \omega_t^*} \right], \\
    %
    $\mathcal{E}(t) = \mathbb{E} \left[ \norm{ \eta_t - \eta_t^* }^2 \right] + \mathbb{E} \left[ \norm{\omega_t - \omega_t^*}^2 \right]$. \label{ineq:actor_1}
%\end{align}
\end{theorem}

We provide a more detailed statement of Theorem \ref{thm:actor_bd} and a complete proof in Theorem \ref{thm:actor_bd_app} in the appendix.
In addition to the $\bo{T^{-1/2}}$ term and the inherent $\bo{\mathcal{E}_{app}}$ bias term, this bound depends on the average value of the critic error via $\mathcal{E}(t)$ and Markovian sampling through $\max_{t \in [T]} \tau_{mix}^{\theta_t} \frac{\log T_{\max}}{T_{\max}}$. As we will see in Theorem \ref{thm:critic_analysis_main_body} in the following subsection, the $\mathcal{E}(t)$ term dies to 0 at a favorable rate. The presence of the Markovian sampling term, however, marks the point where our work departs significantly from previous work.

\textbf{Remark.} Interestingly, we note that the right-hand side of \eqref{actor_adagrad} no longer depends upon the step size rate as in \citet[Theorem 4.5]{wu2020finite} due to the use of our modified Adagrad stepsize in the actor update. This allows us to derive an improved overall sample complexity in Theorem \ref{thm:convergence_rate}.

An important consequence of Theorem \ref{thm:actor_bd} is that the level of bias resulting from Markov sampling can be controlled by choosing $T_{\max}$ appropriately. When the maximum mixing time likely to be encountered during training -- captured here by the term $\max_{t \in [T]} \tau_{mix}^{\theta_t}$, is small -- it makes sense to choose $T_{\max}$ to be relatively small as well. When mixing times are long, on the other hand, choosing $T_{\max}$ accordingly keeps the Markovian sampling bias manageable.

\subsection{Convergence of the Critic}\label{criti_convergence}

We turn next to characterizing the convergence of the critic error term arising in bound \eqref{ineq:actor_1} of Theorem \ref{thm:actor_bd}. Similar to that theorem, the resulting bound expresses critic convergence in terms of mixing times encountered during training as well as MLMC rollout length $T_{\max}$. This result is also where our actor-critic scheme explicitly becomes two-timescale due to our choice of stepsize sequences.

\begin{theorem} \label{thm:critic_analysis_main_body}
Assume $\gamma_t = (1 + t)^{-\nu}, \alpha = \alpha_t' / \sqrt{\sum_{k=1}^t \norm{h^{MLMC}_t}^2}$, and $\alpha_t' = (1 + t)^{-\sigma}$, where $0 < \nu < \sigma < 1$. Then
\small
\begin{align}
    \frac{1}{T} \sum_{t=1}^T  \mathcal{E}(t) \leq &\bo{T^{\nu - 1}} + \bo{T^{-2(\sigma - \nu)}} \nonumber\\
    &+ \wo{ \max_{t \in [T]} \tau_{mix}^{\theta_t} \log T_{\max}} \bo{T^{-\nu}} \nonumber\\
    %
    % &+ \wo{ \sqrt{ \max_{t \in [T]} \tau_{mix}^{\theta_t} \frac{\log T_{\max}}{T_{\max}} } } \\
    %
    &\quad+ \wo{ \max_{t \in [T]} \tau_{mix}^{\theta_t} \frac{\log T_{\max}}{T_{\max}} }. \label{ineq:critic_1}
\end{align}
\normalsize
\end{theorem}
For the proof of Theorem \ref{thm:critic_analysis_main_body}, refer to Theorems \ref{thm:reward_analysis_app} and \ref{thm:critic_bd_app}) in the appendix. Unlike the actor bound \eqref{ineq:actor_1}, the only term in \eqref{ineq:critic_1} that does not diminish with $T$ is the Markovian sampling term containing $\max_{t \in [T]} \tau_{mix}^{\theta_t} \frac{\log T_{\max}}{T_{\max}}$. As in the actor case, this bias can be controlled via the proper selection of $T_{\max}$. As we will see in the final result of this section, this Markovian sampling term will ultimately be absorbed into the analogous term from Theorem \ref{thm:actor_bd}.

\begin{figure*}[t]
  \centering
  \subfigure[]{\includegraphics[width=0.92\columnwidth]{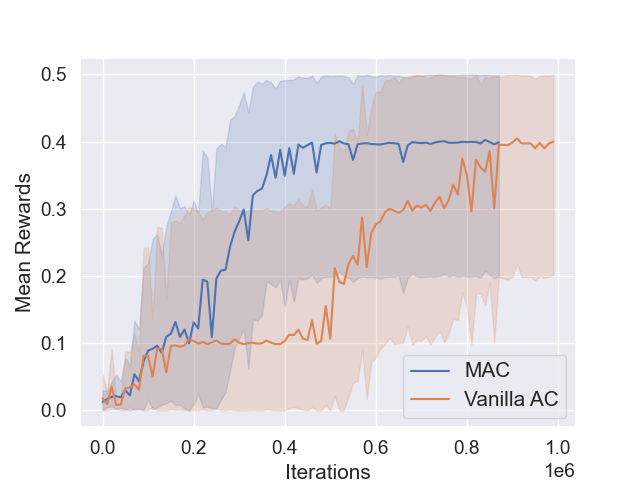}}
    \subfigure[]{\includegraphics[width=0.92\columnwidth]{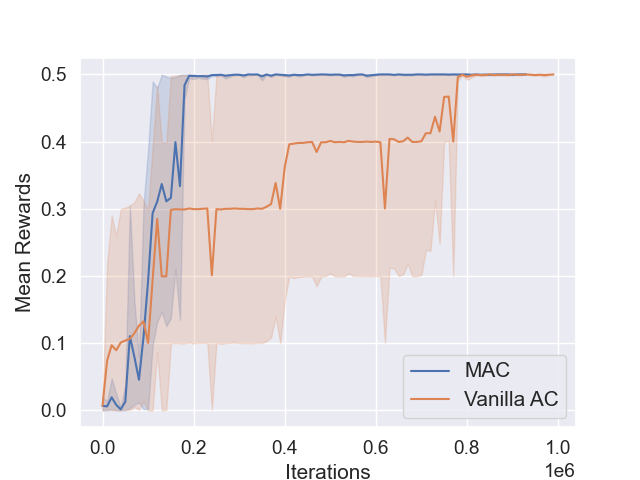}}
  \caption{(a) Mean Rewards over 3 million samples with $T_{\max} = 8$ for MAC and rollout = 3 for Vanilla AC with $6 \times 6$ grid. (b) Mean Rewards over 4 million samples with $T_{\max} = 16$ for MAC and rollout = 4 for Vanilla AC with $10 \times 10$ grid. }
  \label{fig:experiments}
\end{figure*}
\subsection{Convergence Rate and Sample Complexity}

We now present our main result characterizing the convergence rate of Algorithm \ref{alg:PG_MAG} in terms of only the total number of iterations, mixing times encountered and $T_{\max}$ used during training, and the function approximation bias $\mathcal{E}_{app}$. We present the result in Theorem \ref{thm:convergence_rate} next, which follows directly from Theorems \ref{thm:actor_bd} and \ref{thm:critic_analysis_main_body}.
\begin{theorem}(Convergence Rate) \label{thm:convergence_rate}
    Under the assumptions of Theorems \ref{thm:actor_bd} and \ref{thm:critic_analysis_main_body} and with selection $\sigma=0.75$ and $\nu=0.5$, we have
    \begin{align}\label{final_bound}
    \frac{1}{T} \sum_{t=1}^T \mathbb{E}
     \left[ \bignorm{ \nabla J(\theta_t) }^2 \right]\leq & \bo{ \mathcal{E}_{app} }+\wo{ \frac{\tau_{mix} \log T_{\max}}{\sqrt{T}}} \nonumber
    \\
    &\hspace{-4mm} +\wo{ { \frac{\tau_{mix}\log T_{\max}}{T_{\max}}} } ,
    \end{align}
     where $\tau_{mix} :=\max_{t \in [T]} \tau_{mix}^{\theta_t}$. 
\end{theorem}

% \textcolor{red}{Bring the discussion back to Table 1: how do these results compare to the mentioned theorems, especially Xu etal 2020b? }

The proof of Theorem \ref{thm:convergence_rate} is provided in Appendix \ref{prof_thm:convergence_rate}. The result in Theorem \ref{thm:convergence_rate} provides an explicit dependence of the final convergence rate on the maximum mixing time $\tau_{mix}$ encountered during training as well as rollout length $T_{\max}$. The first term is $\bo{ \mathcal{E}_{app} }$ on the right-hand side of \eqref{final_bound} is unavoidable due to the use of linear function approximation for the critic, but can be kept small or even driven to zero with appropriate feature selection. The second term shows the dependence on the mixing rate and shows that we recover the original iid rates if $\tau_{mix}=1$. The last term on the right-hand side of \eqref{final_bound} is interesting because that is the final bias we are incurring due to the use of finite length rollout trajectories $T_{\max}$. If we make $T_{\max}=T$ as in \citet{dorfman2022}, we will recover the rate of $\bo{ \mathcal{E}_{app} }+\wo{ \frac{\tau_{mix} \log T_{\max}}{\sqrt{T}}}$.

We present the sample complexity result next.
\begin{corollary} \label{corollary}
Let us consider $T_{\max}=\sqrt{T}$ and $\mathcal{E}_{app}\leq \epsilon$. Absorbing the logarithmic terms in the $\tilde{\mathcal{O}}$ notation, it holds that to achieve $\min_{1\leq t\leq T} \mathbb{E}
     \left[ \bignorm{ \nabla J(\theta_t) }^2 \right]\leq \epsilon$, we need 
%
%\begin{align}
   $ T\geq \tilde{\mathcal{O}}\left(\frac{\tau_{mix}^2}{\epsilon^2}\right)$.
%\end{align}
\end{corollary}
The proof of Corollary \ref{corollary} follows directly from the statement of Theorem \ref{thm:convergence_rate}. We remark that, even for fast mixing settings where we can ignore the dependence on $\tau_{mix}^2$ in Corollary \ref{corollary}, our proposed algorithm achieves sample complexity $\tilde{\mathcal{O}}\left(\frac{1}{\epsilon^2}\right)$, which improves upon the state of the art result of $\tilde{\mathcal{O}}\left(\frac{1}{\epsilon^{2.5}}\right)$  in \citet{wu2020finite}. This improvement is due to the use of Adagrad step size in the actor update.

\textbf{Remark.} It is interesting to note that the analysis presented in this section recovers results for the simplified i.i.d. sampling setting: since mixing occurs immediately, $\max_{t \in [T]} \tau_{mix}^{\theta_t} = 1$, so we can simply choose $T_{max} = 1$. At the other extreme, when mixing is very slow we intuitively expect that single- or few-sample estimates of the policy gradient like those considered in \cite{wu2020finite, xu2020improving, qiu2021finite, chen2022finite} will be highly inaccurate due to the failure of the fast mixing condition of Assumption 4.2 of \cite{wu2020finite} and Assumption 2 of \cite{xu2020improving}, for example, making a larger number of samples imperative. Theorems \ref{thm:actor_bd}, \ref{thm:critic_analysis_main_body}, and \ref{thm:convergence_rate} are the first results to shed light on this trade-off.

\section{Experiments}\label{sec:experiments}
In this section, we perform preliminary proof of concept experiments to  evaluate the performance of the proposed MAC algorithm and compare it against the vanilla actor-critic.  While we concede that numerous enhancements to actor-critic have been considered, based on Nesterov acceleration \cite{kumar2019sample}, parallelization (Asynchronous Advantage Actor-Critic \cite{mnih2016asynchronous}), and offline processing of prior trajectory information (Soft Actor-Critic \cite{haarnoja2018soft}), our focus is on revealing the experimental dependence of actor-critic's stability on the environment's mixing time. Therefore, for carefully controlled experimentation, we only compare against Vanilla actor-critic as detailed in Sec. \ref{sec:actor_critic}.
We consider an $n \times n$ grid with a starting position at the top left and a goal at the bottom right. There are five actions: stay, up, down, left, and right. An action that results in the goal state gives the agent a $+1$ reward and $+0$ for all other states. In Figure \ref{fig:experiments}, we report algorithm performance in terms of mean reward returns over $5$ trials with $95\%$ confidence intervals. 

We compare MAC against Vanilla AC with a standard gradient estimator. In practice, we use a constant learning rate for the actor, critic, and reward estimation. For comparison, we ran the Vanilla AC for $1$ million iterations setting its constant rollout length to the largest integer under the average rollout length of MAC. For $T_{\max} = 8$, the average rollout length is $3.42$, so the rollout length for Vanilla AC is $3$. 
Thus, $3$ million samples were observed for the Vanilla AC. To have a similar number of observed samples, we ran MAC for $877192$ iterations. Similarly when $T_{\max} = 16$, the average rollout length is $4.26$. Therefore, we ran MAC for $936768$ iterations. The details table of hyperparameters is provided in Appendix \ref{experiments_details}.
In Figure \ref{fig:experiments} (a) we set $n = 6$ and $T_{\max} = 8$ for MAC. For MAC and Vanilla AC, we set the learning rate for actor, critic, and reward estimation to $.01$. In Figure \ref{fig:experiments} (b), $n = 10$ and $T_{\max} = 16$ and learning rate is $.005$. We observe that for both experiments, MAC converges faster to the maximum reward than Vanilla AC, showing MLMC's advantage over a standard gradient estimator.

% \begin{figure}[htbp!]
%     \centering
%       \begin{subfigure}{0.45\textwidth}
%         \includegraphics[width=\textwidth]{tmax8_6x6.png}
%           \caption{Nice image1}
%           \label{fig:tmax8}
%       \end{subfigure}
%       \hfill
%       \begin{subfigure}{0.45\textwidth}
%         \includegraphics[width=\textwidth]{tmax16_10x10_lr005.png}
%           \caption{Nice image 2}
%           \label{fig:tmax16}
%       \end{subfigure}
% \caption{
% \label{fig:NiceImage}%
% Two images}
% \end{figure}

\section{Conclusions and Limitations }
In this work, for the first time, we established the explicit dependence of the convergence rate of the actor-critic algorithm on the mixing time of the underlying Markov transitions induced by the policy. This allows us to remove the fast mixing assumptions in the existing literature and utilize actor-critic algorithms for applications even with slower mixing times which are popular in robotics, finance, etc. To establish the results, we propose a multi-level Monte Carlo-based gradient estimator for the actor, critic, and average reward estimator. This helps to establish the convergence rate in terms of mixing time. 
As a limitation, our current dependence on mixing time is not the sharpest possible. One can further improve the dependence on mixing time from linear to sublinear, which is a valid scope of future research. 
%\clearpage
\bibliography{example_paper}
\bibliographystyle{icml2022}

%%%%%%%%%%%%%%%%%%%%%%%%%%%%%%%%%%%%%%%%%%%%%%%%%%%%%%%%%%%%%%%%%%%%%%%%%%%%%%%
%%%%%%%%%%%%%%%%%%%%%%%%%%%%%%%%%%%%%%%%%%%%%%%%%%%%%%%%%%%%%%%%%%%%%%%%%%%%%%%
% APPENDIX
%%%%%%%%%%%%%%%%%%%%%%%%%%%%%%%%%%%%%%%%%%%%%%%%%%%%%%%%%%%%%%%%%%%%%%%%%%%%%%%
%%%%%%%%%%%%%%%%%%%%%%%%%%%%%%%%%%%%%%%%%%%%%%%%%%%%%%%%%%%%%%%%%%%%%%%%%%%%%%%
\newpage
\onecolumn
\appendix
\addcontentsline{toc}{section}{Appendix} % Add the appendix text to the document TOC
\part{Appendix} \label{appendix}% Start the appendix part
\parttoc % Insert the appendix TOC

\section{Detailed Context of Related Works}\label{Related_Works_appendix}

Actor-critic by  \citet{Konda2000} comprises algorithms that alternate between value function estimation (critic) and policy search updates (actor), which may be seen as a form of policy iteration \cite{bertsekas2011approximate} that incorporates stochastic approximation \cite{borkar1997actor}. We discuss each facet separately, before launching into their fusion.

%\textbf{AC}  introduced the AC algorithm. % In contrast to the aforementioned studies of the AC algorithm, our study focuses on the Markovian sampling in the average reward case without any assumptions on mixing time and establishes SOTA AC convergence rate.

\textbf{TD Learning} To evaluate the policy update direction, an estimate of the value function is required. To compute this estimate, stochastic fixed point iterations are considered to solve Bellman's equation \citet{sutton1988}, whose stability under linear function approximation was established in \citet{roy1997}. Since then, a plethora of works has studied the stability properties of TD-based policy evaluation. Initially, their asymptotic convergence was prioritized \cite{tadic2001convergence}, but more recently, non-asymptotic results have gained salience. For discounted TD with Markovian samples, \citet{bhandari2018} established finite-time convergence bounds which scale linearly with mixing time $\tau_{mix}$. \citet{dorfman2022} then improved the rate to be proportional to the $\sqrt{\tau_{mix}}$ using a multi-level gradient estimator and adaptive learning rate. \citet{qiu2021} studied TD under the average reward setting, which also imposes exponentially fast mixing that manifests in an additional logarithmic term in the sample complexity. These results all hinge upon imposing restrictive conditions on the mixing time.

\textbf{Policy Gradient} With a value function estimate in hand, one can multiple this quantity together with the gradient of the log-likelihood of a policy, i.e., the score function, to evaluate an estimate of the policy gradient \cite{williams1992simple,sutton1999policy}. Then, gradient ascent steps are taken with respect to policy parameters. The convergence of policy gradient has been studied extensively. Similar to TD, early work \cite{borkar2000ode} focused on asymptotic stability via tools from dynamical systems \cite{borkar2000ode}. More recently, its sample complexity has been established for a variety of settings: for tabular \cite{bhandari2019global,agarwal2020optimality} and softmax policies \cite{mei2020global}, rates to global optimality exist. For general parameterized policies, early works focused on ``policy improvement" bounds \cite{pirotta2013adaptive,pirotta2015policy}, and more recently, rates towards stationarity \cite{bedi2022hidden} and local extrema \cite{zhang2020global} have been studied, and under special neural architectures, globally optimal solutions \cite{wang2019neural,leahy2022convergence} are achievable. This topic is an active area of work, and covering all related sub-topics is beyond our scope. We merely identify that these performance certificates all hinge upon the mixing time of the induced Markov chain going to null exponentially fast.

\textbf{Actor-Critic} As previously mentioned, the stability of actor-critic was initially focused on asymptotics \cite{borkar1997actor}. More recently, its non-asymptotic rate has been derived under i.i.d. assumptions \cite{kumar2019sample,wang2019neural}, and more recently under a variety of different types of Markovian data -- see Table \ref{table_mixing}. However, these results impose that any temporal correlation of data across time vanishes exponentially fast as quantified by the mixing rate. In this way, we are able to match \cite{chen2022finite} but without this restriction. 

%%%%%%%%%%%%%%%%%%%%%% ANALYSIS STARTS HERE

\section{Preliminaries}

% \subsection{Preliminary Results}
%
Before proceeding with our analysis of Algorithm \ref{alg:PG_MAG}, we need some preliminary results and assumptions.

\subsection{Preliminary Results}
The statements of the results in this section have been adapted from \cite{dorfman2022} to fit the setting considered in our paper. Except in the case of Lemma \ref{lemma:31_app}, their proofs follow directly from that work. First, we need the following concentration bound concerning gradient estimation from Markovian data.
\begin{lemma}{Lemma A.5, \cite{dorfman2022}.} \label{lemma:a5}
Fix $K, N \in \mathbb{N}$ such that $N \geq 2K$. Let a policy parameter $\theta_t \in \Theta$ be given, and fix a trajectory $z_t = \{ z_t^i = (s_t^i, a_t^i, r_t^i, s_t^{i+1}) \}_{i \in [N]}$ generated by following policy $\pi_{\theta_t}$ starting from $s_t^0 \sim \mu_0(\cdot)$. Let $\nabla L(x)$ be a gradient that we wish to estimate over $z_t$, where $\mathbb{E}_{z \sim \mu_{\theta_t}, \pi_{\theta_t}} \left[ l(x, z) \right] = \nabla L(x)$, and $x \in \mathcal{K} \subset \mathbb{R}^k$ is the parameter of the estimator $l$, i.e., $x_t = \theta_t, \eta_t$, or $\omega_t$. Finally, assume that $\norm{ l(x, z) }, \norm{ \nabla L(x) } \leq G_L$, for all $x \in \mathcal{K}, z \in \mathcal{S} \times \mathcal{A} \times \mathbb{R} \times \mathcal{S}$. Then, for every $\delta > N d_{\text{mix}}(K)$ and every $x_t \in \mathcal{K}$ measurable w.r.t. $\mathcal{F}_{t-1} = \sigma(\theta_k, \eta_k, \omega_k, z_k; k \leq t-1)$, we have
\begin{equation}
    \mathbb{P}_{t-1} \left( \bignorm{ \frac{1}{N} \sum_{i=1}^N l(x_t, z_t^i) - \nabla L(x_t) } \leq 12G_L \sqrt{ \frac{K}{N} } \left( 1 + \sqrt{ \log (K / \tilde{\delta} ) } \right) + \frac{ 6GK }{ N } \right) \geq 1 - \delta,
\end{equation}
where $\tilde{\delta} = \delta - N d_{\text{mix}}(K)$.
\end{lemma}
We will use this result to facilitate our analyses of each of the MLMC estimators $f_t^{MLML}, g_t^{MLMC}, l_t^{MLMC}$ used in Algorithm \ref{alg:PG_MAG}. We also need the following error bound, which follows from Lemma \ref{lemma:a5}.
\begin{lemma}{Lemma A.6, \cite{dorfman2022}.} \label{lemma:a6}
Let $\nabla L, l, z_t$ be as in Lemma \ref{lemma:a5}. Define $l_t^N = \frac{1}{N} \sum_{i=1}^N l(x_t, z_t^i)$. Fix $T_{max} \in \mathbb{N}$ and let $K = \tau_{max}^{\theta_t} \lceil 2 \log T_{max} \rceil$. Then, for every $N \in \left[ T_{max} \right]$ and every $x_t \in \mathcal{K}$ measurable w.r.t. $\mathcal{F}_{t-1}$,
\begin{align}
    \mathbb{E} \left[ \norm{ l_t^N - \nabla L(x_t) } \right] &\leq O \left( G_L \sqrt{ \log KN } \sqrt{ \frac{K}{N} } \right), \\
    \mathbb{E} \left[ \norm{ l_t^N - \nabla L(x_t) }^2 \right] &\leq O \left( G_L^2 \log (KN) \frac{K}{N} \right).
\end{align}
\end{lemma}

The following important result establishes key properties of MLMC estimators. It is an extension of Lemma 3.1 from \cite{dorfman2022}, clarifying the effect of using rollout length $T_{max}$ in the MLMC estimator.
\begin{lemma} \label{lemma:31_app}
Let $\nabla L, l, z_t$ be as in Lemma \ref{lemma:a5}. Let $J_t \sim \text{Geom}(1/2)$. Define the MLMC estimator
\begin{equation}\label{l_MLMC_Gradient}
l_t^{MLMC} = l_t^0 +
\begin{cases}
2^{J_t}(l_t^{J_t} - l_t^{J_t - 1}),& \text{if } 2^{J_t}\geq T_{max}, \\
0,              & \text{otherwise.}
\end{cases}             
\end{equation}

Let $j_{max} = \floor{\log T_{max}}$. Fix $x_t$ measurable w.r.t. $\mathcal{F}_{t-1}$. Assume $T_{max} \geq \tau_{mix}^{\theta_t}$, $\norm{\nabla L(x)} \leq G_L$, for all $x \in \mathcal{K}$, and $\norm{ l_t^N } \leq G_L$, for all $N \in \left[ T_{max} \right]$. Then
\begin{align}
    \mathbb{E}_{t-1} \left[ l_t^{MLMC} \right] &= \mathbb{E}_{t-1} \left[ l_t^{j_{max}} \right], \label{eqn:31_mean} \\
    \mathbb{E} \left[ \norm{ l_t^{MLMC} }^2 \right] &\leq \wo{ G_L^2 \tau_{mix}^{\theta_t} \log T_{max} }. \label{eqn:31_2nd_moment}
\end{align}
\end{lemma}
\begin{proof}
For brevity, let $l_t := l_t^{MLMC}$. To show \eqref{eqn:31_mean}, we simply recall that $l_t = l_t^0 + 2^{J_t} \left( l_t^{J_t} - l_t^{J_t - 1} \right)$ and note that
\begin{equation}
    \mathbb{E}_{t-1} \left[ l_t \right] = \mathbb{E}_{t-1} \left[ l_t^0 \right] + \sum_{i=1}^{j_{max}} P(J_t = j) 2^j \mathbb{E}_{t-1} \left[ l_t^j - l_t^{j-1} \right] = \mathbb{E}_{t-1} \left[ l_t^{j_{max}} \right].
\end{equation}
For \eqref{eqn:31_2nd_moment}, first note that by Cauchy-Schwarz and boundedness of $l_t^j$, for all $j \in \left[ T_{max} \right]$, we know that
\begin{equation} \label{eqn:31_1}
    \mathbb{E} \left[ \norm{ l_t }^2 \right] \leq 2 \mathbb{E} \left[ \norm{ l_t - l_t^0 }^2 \right] + 2G_L^2.
\end{equation}
Now, since $l_t = l_t^0 + 2^{J_t} \left( l_t^{J_t} - l_t^{J_t - 1} \right)$,
\begin{align}
    \mathbb{E} \left[ \norm{ l_t - l_t^0 }^2 \right] &= \sum_{j=1}^{j_{max}} P(J_t = j) \mathbb{E} \left[ \bignorm{ 2^j \left( l_t^j - l_t^{j-1} \right) }^2 \right] \\
    &= \sum_{j=1}^{j_{max}} 2^j \mathbb{E} \left[ \bignorm{ \left( l_t^j - l_t^{j-1} \right) }^2 \right] \\
    &\leq \sum_{j=1}^{j_{max}} 2^j \left( 2 \mathbb{E} \left[ \bignorm{ l_t^j - \nabla J(\theta_t) }^2 \right] + 2 \mathbb{E} \left[ \bignorm{l_t^{j-1} - \nabla J(\theta_t)}^2 \right] \right) 
    \\
    &\labelrel{\leq}{ineq:31_1} \sum_{j=1}^{j_{max}} 2^j \left( \wo{ \frac{1}{2^j} G_L^2 \tau_{mix}^{\theta_t} \log(T_{max}) } \right) \end{align}
    \begin{align}
    &= \sum_{j=1}^{j_{max}} \wo{ G_L^2 \tau_{mix}^{\theta_t} \log T_{max} } \\
    &= \wo{ G_L^2 \tau_{mix}^{\theta_t} \log T_{max} }, \label{eqn:31_2}
\end{align}
where \eqref{ineq:31_1} follows from Lemma \ref{lemma:a6} and \eqref{eqn:31_2} holds by the definition of $j_{max}$. Combining \eqref{eqn:31_1} with \eqref{eqn:31_2} gives the result.
\end{proof}

Finally, we will use the following result to manipulate the AdaGrad stepsizes in the final result of this section.
\begin{lemma}{Lemma 4.2, \cite{dorfman2022}.} \label{lemma:42}
For any non-negative real numbers $\{ a_i \}_{i \in [n]}$,
\begin{equation}
    \sum_{i=1}^n \frac{a_i}{ \sqrt{ \sum_{j=1}^i a_j } } \leq 2 \sqrt{ \sum_{i=1}^n a_i }.
\end{equation}
\end{lemma}

\subsection{Assumptions}

We will also need the following assumptions.

\begin{assumption} \label{assum:pg_objective}
The objective $J(\theta)$ is $L$-Lipschitz in $\theta$. There exists $G_H$ such that $\norm{ \nabla J(\theta) } \leq G_H$, for all $\theta$.
\end{assumption}

\begin{assumption} \label{assum:omega_projection}
The critic update includes a projection onto the ball of radius $R_{\omega}$ about the origin.
\end{assumption}

\begin{assumption} \label{assum:positive_definite}
For each $\theta$, the matrix $A_{\theta} = \E_{s \sim \mu_{\theta}, a \sim \pi_{\theta}, s' \sim p(\cdot | s, a)} \left[ \phi(s) (\phi(s) - \phi(s'))^T \right]$ is positive definite.
\end{assumption}

\section{Convergence Analysis of Actor} \label{sec:actor_analysis}
In this section, we provide a bound on the average policy gradient norm achieved by Algorithm \ref{alg:PG_MAG}, leveraging the MLMC analysis machinery of \cite{dorfman2022} to reveal dependence on the worst-case mixing time encountered during training. Combined with the error analysis of Section \ref{sec:rew_and_critic_analysis}, this forms the core of our analysis of Algorithm \ref{alg:PG_MAG}. The analysis largely follows that of \cite{dorfman2022}, with key modifications to accommodate the \textit{average reward estimation}, \textit{critic estimation}, and \textit{critic function approximation bias} inherent in the average-reward actor-critic setting. 

% We remark that in our convergence analysis shows explicit dependence on the trajectory length $T_{\max}$. \textcolor{red}{TODO: describe ramifications of showing explicit dependence on $T_{\max}$ and the trade-off between $\tau_{mix}$ and $T_{\max}$; this generalizes/clarifies the original Dorfman analysis.}

% \subsection{Actor Bound}
%
As the first step in our actor analysis, we prove a version of Lemma \ref{lemma:a6} that incorporates average reward estimation error and critic error. Before starting the result and its proof, we develop some notation to facilitate the exposition. Let
\begin{align}
	\nabla J_t^{i} = & \left(r_t^i - \eta_t + \langle\phi(s_{t}^{i+1}),\omega_t\rangle - \langle\phi(s_{t}^{i}),\omega_t\rangle\right) \nabla \log \pi_{\theta_t}\left(a_t^i|s_t^i\right)\label{J1},
	\\
	\nabla J_t^{i,\eta} = & \left(r_t^i - \eta_t^* + \langle\phi(s_{t}^{i+1}),\omega_t\rangle - \langle\phi(s_{t}^{i}),\omega_t\rangle\right) \nabla \log \pi_{\theta_t}\left(a_t^i|s_t^i\right)\label{J2},
		\\
	\nabla J_t^{i,\eta, \omega} = & \left(r_t^i - \eta_t^* + \langle\phi(s_{t}^{i+1}),\omega_t^*\rangle - \langle\phi(s_{t}^{i}),\omega_t^*\rangle\right) \nabla \log \pi_{\theta_t}\left(a_t^i|s_t^i\right)\label{J3},
			\\
	\nabla J_t^{i,\eta,V} = & \left(r_t^i - \eta_t^* + V_{\theta_t}(s_{t}^{i+1}) -V_{\theta_t}(s_{t}^{i}) \right)\nabla \log \pi_{\theta_t}\left(a_t^i|s_t^i\right)\label{J4},
\end{align}
where $\eta_t^* = J(\theta_t)$ and $\omega_t^*$ is the limiting point of TD(0) applied to evaluating the policy $\pi_{\theta_t}$. Notice that
\begin{align}
\nabla J_t^{i}- \nabla J(\theta_t)= \big(\underbrace{\nabla J_t^{i}- \nabla J_t^{i,\eta}}_{(a)}\big) + \big(\underbrace{\nabla J_t^{i,\eta}-	\nabla J_t^{i,\eta,\omega}}_{(b)}\big) + \big(\underbrace{	\nabla J_t^{i,\eta, \omega}-	\nabla J_t^{i,\eta,V}}_{(c)}\big) + \big(\underbrace{	\nabla J_t^{i,\eta,V}- \nabla J(\theta_t)}_{(d)}\big),
\end{align}
where
\begin{align}
	&\text{(a):}  \ \ \nabla J_t^{i}- \nabla J_t^{i,\eta} = \left(\eta_t^*-\eta_t\right) \nabla \log \pi_{\theta_t}\left(a_t^i|s_t^i\right)\label{expressions_1}
	\\
	&\text{(b):} \ \ \nabla J_t^{i,\eta}-	\nabla J_t^{i,\eta,w} = \langle\phi(s_t^{i+1})-\phi(s_t^{i}), \omega_t-\omega_t^*\rangle \nabla \log \pi_{\theta_t}\left(a_t^i|s_t^i\right)
	\\
	&\text{(c):} \ \ 	\nabla J_t^{i,\eta,w}-	\nabla J_t^{i,\eta,V} = \left[\left(\langle\phi(s_t^{i+1}), \omega_t^*\rangle-V_{\theta_t}(s_t^{i+1})\right)-\left(\langle\phi(s_t^{i}), \omega_t^*\rangle-V_{\theta_t}(s_t^{i})\right)\right]\nabla \log \pi_{\theta_t}\left(a_t^i|s_t^i\right)
	% \\
	% &\text{(d):} \ \  	\nabla J_t^{i,\eta,V}- \nabla J(\theta_t) =  \red{\text{????}}\label{expressions_4} .
\end{align}
and, since $\mathbb{E}_{\mu_{\theta_t}, \pi_{\theta_t}} \left[ \nabla J_t^{i,\eta,V} \right] = \nabla J(\theta_t)$, (d) is the error between $\nabla J(\theta_t)$ and the ideal policy gradient estimator. Define
\begin{align}
	\mathcal{E}_{\text{app}} := \sup_{s,\theta} |\langle\phi(s),\omega(\theta) - V_{\theta}(s)\rangle|, \hspace{4mm} C:=\sup_{s,s'} \|\phi(s)-\phi(s')\|,
\end{align}
and let $B > 0$ be such that
\begin{align}
    \sup_{\theta,a,s} \norm{\nabla \log \pi_{\theta}(a|s)} \leq B.
\end{align}
\begin{lemma} \label{lemma:a6ss}
Assume $\norm{ \nabla J(\theta) }, \norm{ \nabla J_t^{i, \eta, V} } \leq G_H$, for all $\theta, s_t^i, a_t^i$. Fix $T_{max} \in \mathbb{N}$ and let $K = \tau_{max}^{\theta_t} \lceil 2 \log T_{max} \rceil$. Define $h_t^N = \frac{1}{N} \sum_{i=1}^N \nabla J_t^i$, for $N \in \left[ T_{max} \right]$. Then, for all $N \in \left[ T_{max} \right]$ and $\theta_t$ measurable w.r.t. $\mathcal{F}_{t-1}$,
\begin{align}
    \mathbb E \left[ \norm{h_t^N - \nabla J(\theta_t)} \right] &\leq O \left( G_H \sqrt{\log KN} \sqrt{\frac{K}{N}} \right) + \mathcal{E}_1(t) + 2B \mathcal{E}_{app}, \\
    \mathbb E \left[ \norm{ h_t^N - \nabla J(\theta_t) }^2 \right] &\leq O \left( G_H^2 \log(KN)\frac{K}{N} \right) + \mathcal{E}_2(t) + 16 B^2 \mathcal{E}_{app},
\end{align}
where
\begin{align}
    \mathcal{E}_1(t) &= B \mathbb{E} \left[ \norm{ \eta_t - \eta_t^* } \right] + BC \mathbb{E} \left[ \norm{\omega_t - \omega_t^*} \right], \label{eqn:a6ss_1} \\
    \mathcal{E}_2(t) &= 4B^2 \mathbb{E} \left[ \norm{ \eta_t - \eta_t^* }^2 \right] + 4B^2C^2 \mathbb{E} \left[ \norm{\omega_t - \omega_t^*}^2 \right]. \label{eqn:a6ss_2}
\end{align}
\end{lemma}
\begin{proof}
First notice that
\begin{align}
    \bignorm{ h_t^N - \nabla J(\theta_t) } &\leq \bignorm{ \frac{1}{N} \sum_{i=1}^N \nabla J_t^{i, \eta, V} - \nabla J(\theta_t) } + \bignorm{ \frac{1}{N} \sum_{i=1}^N \nabla J_t^i - \nabla J_t^{i, \eta} } \\
    &\hspace{1cm} + \bignorm{ \frac{1}{N} \sum_{i=1}^N \nabla J_t^{i, \eta} - \nabla J_t^{i, \eta, \omega} } + \bignorm{ \frac{1}{N} \sum_{i=1}^N \nabla J_t^{i, \eta, \omega} - \nabla J_t^{i, \eta, V} } \\
    &\leq \bignorm{ \frac{1}{N} \sum_{i=1}^N \nabla J_t^{i, \eta, V} - \nabla J(\theta_t) } + B \norm{ \eta_t - \eta_t^* } + BC \norm{ \omega_t - \omega_t^* } + 2B \mathcal{E}_{app}.
\end{align}
As a consequence, we also have
\begin{equation}
    \bignorm{ h_t^N - \nabla J(\theta_t) }^2 \leq 4 \bignorm{ \frac{1}{N} \sum_{i=1}^N \nabla J_t^{i, \eta, V} - \nabla J(\theta_t) }^2 + 4B^2 \norm{ \eta_t - \eta_t^* }^2 + 4B^2C^2 \norm{ \omega_t - \omega_t^* }^2 + 16B^2 \mathcal{E}_{app}^2.
\end{equation}
Taking expectations and applying Lemma \ref{lemma:a6} with $x_t = \theta_t$, $l(\theta_t, z_t^i) = \nabla J_t^{i, \eta, V}$, $\nabla L(\theta_t) = \nabla J(\theta_t)$ yields the result.
\end{proof}

We next prove a key result regarding the bias and second moment of our policy gradient estimate. It is a generalization of Lemma 3.1 in \cite{dorfman2022} building on our Lemma \ref{lemma:a6ss}.
\begin{lemma} \label{lemma:31ss}
Let $j_{max} = \floor{\log T_{max}}$ in Algorithm \ref{alg:PG_MAG}. Fix $\theta_t$ measurable w.r.t. $\mathcal{F}_{t-1}$. Assume $T_{max} \geq \tau_{mix}^{\theta_t}$, $\norm{\nabla J(\theta)} \leq G_H$, for all $\theta$, and $\norm{ h_t^N } \leq G_H$, for all $N \in \left[ T_{max} \right]$. Then
\begin{align}
    \mathbb{E}_{t-1} \left[ h_t^{MLMC} \right] &= \mathbb{E}_{t-1} \left[ h_t^{j_{max}} \right], \label{eqn:31ss_mean} \\
    \mathbb{E} \left[ \norm{ h_t^{MLMC} }^2 \right] &\leq \wo{ G_H^2 \tau_{mix}^{\theta_t} \log T_{max} } + 8 \log(T_{max}) T_{max} \left( \mathcal{E}_2(t) + 16B^2 \mathcal{E}_{app}^2 \right). \label{eqn:31ss_2nd_moment}
\end{align}
\end{lemma}

\begin{proof}
For brevity, let $h_t := h_t^{MLMC}$. Equation \eqref{eqn:31ss_mean} follows directly from Lemma \ref{lemma:31_app}. For \eqref{eqn:31ss_2nd_moment}, first note that by Cauchy-Schwarz and boundedness of $h_t^j$, for all $j \in \left[ T_{max} \right]$, we know that
\begin{equation} \label{eqn:31ss_1}
    \mathbb{E} \left[ \norm{ h_t }^2 \right] \leq 2 \mathbb{E} \left[ \norm{ h_t - h_t^0 }^2 \right] + 2G_H^2.
\end{equation}
Now, since $h_t = h_t^0 + 2^{J_t} \left( h_t^{J_t} - h_t^{J_t - 1} \right)$,
\begin{align}
    \mathbb{E} \left[ \norm{ h_t - h_t^0 }^2 \right] &= \sum_{j=1}^{j_{max}} P(J_t = j) \mathbb{E} \left[ \bignorm{ 2^j \left( h_t^j - h_t^{j-1} \right) }^2 \right] \\
    &= \sum_{j=1}^{j_{max}} 2^j \mathbb{E} \left[ \bignorm{ \left( h_t^j - h_t^{j-1} \right) }^2 \right] \\
    &\leq \sum_{j=1}^{j_{max}} 2^j \left( 2 \mathbb{E} \left[ \bignorm{ h_t^j - \nabla J(\theta_t) }^2 \right] + 2 \mathbb{E} \left[ \bignorm{h_t^{j-1} - \nabla J(\theta_t)}^2 \right] \right). \end{align}
    Next, we can write
    \begin{align}
    \mathbb{E} \left[ \norm{ h_t - h_t^0 }^2 \right] &\labelrel{\leq}{ineq:31ss_1} \sum_{j=1}^{j_{max}} 2^j \left( \wo{ \frac{1}{2^j} G_H^2 \tau_{mix}^{\theta_t} \log(T_{max}) } + 4 \mathcal{E}_2(t) + 16B^2 \mathcal{E}_{app}^2 \right) \\
    &= \sum_{j=1}^{j_{max}} \left( \wo{ G_H^2 \tau_{mix}^{\theta_t} \log T_{max} } + 4 \cdot 2^j \left[ \mathcal{E}_2(t) + 16 B^2 \mathcal{E}_{app}^2 \right] \right) \\
    &\labelrel{\leq}{ineq:31ss_2} \log T_{max} \left( \wo{ G_H^2 \tau_{mix}^{\theta_t} \log T_{max} } + 4 T_{max} \left[ \mathcal{E}_2(t) + 16B^2 \mathcal{E}_{app}^2 \right] \right) \\
    &= \wo{ G_H^2 \tau_{mix}^{\theta_t} \log T_{max} } + 4 \log (T_{max}) T_{max} \left[ \mathcal{E}_2(t) + 16B^2 \mathcal{E}_{app}^2 \right], \label{eqn:31ss_2}
\end{align}
where \eqref{ineq:31ss_1} follows from Lemma \ref{lemma:a6ss} and \eqref{ineq:31ss_2} holds by the definition of $j_{max}$. Combining \eqref{eqn:31ss_1} with \eqref{eqn:31ss_2} gives the result.
\end{proof}

Before proceeding to the final policy gradient norm bound of our actor analysis, we need one additional auxiliary result.
\begin{lemma} \label{lemma:41ss}
    Assume $J(\theta)$ is $L$-smooth. Let $\Delta_t = \sup_{\theta} J(\theta) - J(\theta_t)$ and $\Delta^T_{max} = \max_{t \in \left[ T \right]} \Delta_t$. Then
    \begin{equation}
        \sum_{t=1}^T \bignorm{\nabla J(\theta_t)}^2 \leq \frac{\Delta^T_{max}}{\alpha_T} + \frac{L}{2} \sum_{t=1}^T \alpha \norm{ h_t^{MLMC} }^2 + \sum_{t=1}^T \langle \nabla J(\theta_t) - h_t^{MLMC}, \nabla J(\theta_t) \rangle.
    \end{equation}
\end{lemma}
\begin{proof}
Once again, write $h_t := h_t^{MLMC}$ for brevity. We first have
\begin{align}
    J(\theta_{t+1}) &\geq J(\theta_t) + \alpha_t \nabla J(\theta_t)^T h_t - \frac{L \alpha_t^2}{2} \norm{ h_t }^2 \\
    &= J(\theta_t) + \alpha_t \norm{ \nabla J(\theta_t) }^2 - \alpha_t \langle \nabla J(\theta_t) - h_t, \nabla J(\theta_t) \rangle - \frac{L \alpha_t^2}{2} \norm{ h_t }^2,
\end{align}
where the first equality holds from the smoothness of $J(\theta)$ and the fact that $\theta_{t+1} = \theta_t + \alpha_t h_t$. Rearranging gives
\begin{equation}
    \norm{ \nabla J(\theta_t) }^2 \leq \frac{J(\theta_{t+1}) - J(\theta_t)}{\alpha_t} + \frac{L \alpha_t}{2} \norm{ h_t }^2 + \langle \nabla J(\theta_t) - h_t, \nabla J(\theta_t) \rangle,
\end{equation}
and summing yields
\begin{align}
    \sum_{t=1}^T \norm{ \nabla J(\theta_t) }^2 &\leq \sum_{t=1}^T \frac{ \Delta_t - \Delta_{t+1}}{ \alpha_t } + \frac{L}{2} \sum_{t=1}^T \alpha_t \norm{ h_t }^2 + \sum_{t=1}^T \langle \nabla J(\theta_t) - h_t, \nabla J(\theta_t) \rangle \\
    &\leq \sum_{t=1}^T \frac{ \Delta^T_{max} }{\alpha_T} + \frac{L}{2} \sum_{t=1}^T \alpha_t \norm{ h_t }^2 + \sum_{t=1}^T \langle \nabla J(\theta_t) - h_t, \nabla J(\theta_t) \rangle.
\end{align}
\end{proof}

We are now ready to prove the main result of this section.
\begin{theorem} \label{thm:actor_bd_app}
Assume $J(\theta)$ is $L$-smooth, $\sup_{\theta} | J(\theta) | \leq M$, and $\norm{ \nabla J(\theta) }, \norm{ h_t^{MLMC} } \leq G_H$, for all $\theta, t$. Let $\alpha_t = \alpha_t' / \sqrt{ \sum_{t=1}^T \norm{h_t^{MLMC}}^2 }$, where $\{ \alpha_t' \}$ is an auxiliary stepsize sequence with $\alpha_t' \leq 1$, for all $t \geq 1$. Then
\begin{align}
    \frac{1}{T} \sum_{t=1}^T \mathbb{E} \left[ \bignorm{ \nabla J(\theta_t) }^2 \right] &\leq \wo{ (M + L) G_H \frac{1}{\sqrt{T}} \sqrt{ \max_{t \in [T]} \tau_{mix}^{\theta_t} \log T_{max} } } \\
    &+ \frac{2M + L}{T} \sqrt{ \sum_{t=1}^T 8 \log (T_{max}) T_{max} \left( \mathcal{E}_2(t) + 16B^2 \mathcal{E}_{app}^2 \right) } \\
    &+ \wo{G_H^2 \max_{t \in [T]} \tau_{mix}^{\theta_t} \frac{\log T_{max}}{T_{max}}} + \frac{1}{T} \sum_{t=1}^T \mathcal{E}_2(t)+ 16B^2 \mathcal{E}_{app}^2.
\end{align}
\begin{proof}
Again let $h_t := h_t^{MLMC}$. We have
\begin{align}
    \sum_{t=1}^T \norm{ \nabla J(\theta_t) }^2 &\labelrel{\leq}{ineq:actor_bd_1} \Delta_{max} \sqrt{\sum_{t=1}^T \norm{h_t}^2} + \frac{L}{2} \sum_{t=1}^T \frac{ \alpha_t' \norm{ h_t }^2 }{\sqrt{ \sum_{k=1}^t \norm{ h_k }^2 }} + \sum_{t=1}^T \langle \nabla J(\theta_t) - h_t, \nabla J(\theta_t) \rangle \\
    &\labelrel{\leq}{ineq:actor_bd_2} \Delta_{max} \sqrt{\sum_{t=1}^T \norm{h_t}^2} + \frac{L}{2} \sum_{t=1}^T \frac{ \norm{ h_t }^2 }{\sqrt{ \sum_{k=1}^t \norm{ h_k }^2 }} + \sum_{t=1}^T \langle \nabla J(\theta_t) - h_t, \nabla J(\theta_t) \rangle \\
    &\labelrel{\leq}{ineq:actor_bd_3} ( \Delta_{max} + L) \sqrt{\sum_{t=1}^T \norm{h_t}^2} + \sum_{t=1}^T \langle \nabla J(\theta_t) - h_t, \nabla J(\theta_t) \rangle,
\end{align}
where \eqref{ineq:actor_bd_1} follows from Lemma \ref{lemma:41ss}, inequality \eqref{ineq:actor_bd_2} by the definition of $\alpha_t$, and \eqref{ineq:actor_bd_3} is by Lemma \ref{lemma:42}. This implies that
\begin{align}
    \sum_{t=1}^T \E \left[ \norm{ \nabla J(\theta_t) }^2 \right] &\labelrel{\leq}{ineq:actor_bd_4} \E \left[ ( \Delta_{max} + L) \sqrt{\sum_{t=1}^T \norm{h_t}^2} \right] + \sum_{t=1}^T \E \left[ \langle \nabla J(\theta_t) - h_t^{j_{max}}, \nabla J(\theta_t) \rangle \right] \\
    &\labelrel{\leq}{ineq:actor_bd_5} \E \left[ ( \Delta_{max} + L) \sqrt{\sum_{t=1}^T \norm{h_t}^2} \right] + \sum_{t=1}^T \E \left[ \norm{\nabla J(\theta_t) - h_t^{j_{max}}} \cdot \norm{ \nabla J(\theta_t) } \right] \\
    &\labelrel{\leq}{ineq:actor_bd_6} \E \left[ ( \Delta_{max} + L) \sqrt{\sum_{t=1}^T \norm{h_t}^2} \right] + \sum_{t=1}^T \left( \E \left[ \norm{\nabla J(\theta_t) - h_t^{j_{max}}}^2 \right] \right)^{1/2} \left( \E \left[ \norm{ \nabla J(\theta_t) }^2 \right] \right)^{1/2} \\
    &\labelrel{\leq}{ineq:actor_bd_7} \E \left[ ( \Delta_{max} + L) \sqrt{\sum_{t=1}^T \norm{h_t}^2} \right] + \left( \sum_{t=1}^T \E \left[ \norm{\nabla J(\theta_t) - h_t^{j_{max}}}^2 \right] \right)^{1/2} \left( \sum_{t=1}^T \E \left[ \norm{ \nabla J(\theta_t) }^2 \right] \right)^{1/2},
\end{align}
where \eqref{ineq:actor_bd_4} follows from the law of total expectation, the fact that $\theta_t, \theta_t^*$ are deterministic conditioned on $\mathcal{F}_{t-1}$, and Lemma \ref{lemma:31ss}, \eqref{ineq:actor_bd_5} follows by Cauchy-Schwarz, and \eqref{ineq:actor_bd_5} and \eqref{ineq:actor_bd_6} by applications of H\"{o}lder's inequality.  Define
\begin{align}
    A(T) &= \E \left[ ( \Delta_{max} + L) \sqrt{\sum_{t=1}^T \norm{h_t}^2} \right], \\
    B(T) &= \frac{1}{4} \sum_{t=1}^T \E \left[ \norm{\nabla J(\theta_t) - h_t^{j_{max}}}^2 \right], \\
    C(T) &= \sum_{t=1}^T \E \left[ \norm{ \nabla J(\theta_t) }^2 \right].
\end{align}
The foregoing inequality becomes
\begin{equation}
    C(T) \leq A(T) + 2 \sqrt{ B(T) } \sqrt{ C(T) }
\end{equation}
Consider the following chain of implications:
\begin{align}
    C(T) \leq A(T) + 2 \sqrt{ B(T) } \sqrt{ C(T) } &\Longrightarrow \left( \sqrt{C(T)} - \sqrt{B(T)} \right)^2 \leq A(T) + B(T) \\
    &\Longrightarrow \sqrt{C(T)} - \sqrt{B(T)} \leq \sqrt{A(T)} + \sqrt{B(T)} \\
    &\Longrightarrow \sqrt{C(T)} \leq \sqrt{A(T)} + 2 \sqrt{B(T)} \\
    &\Longrightarrow C(T) \leq 2 A(T) + 8 B(T).
\end{align}
We therefore have
\begin{align}
    \sum_{t=1}^T \E \left[ \norm{ \nabla J(\theta_t) }^2 \right] &\leq 2 \E \left[ ( \Delta_{max} + L) \sqrt{\sum_{t=1}^T \norm{h_t}^2} \right] + 2 \sum_{t=1}^T \E \left[ \norm{\nabla J(\theta_t) - h_t^{j_{max}}}^2 \right] \\    
\end{align}
Now,
\begin{align}
    \mathbb{E} &\left[ \left( \Delta_{max} + L \right) \sqrt{\sum_{t=1}^T \norm{h_t}^2} \right] \labelrel{\leq}{ineq:actor_bd_4} \left( 2M + L \right) \sqrt{ \sum_{t=1}^T \mathbb{E} \left[ \norm{h_t}^2 \right] } \\
    &\labelrel{\leq}{ineq:actor_bd_5} \left( 2M + L \right) \sqrt{ \wo{T G_H^2 \max_{t \in [T]} \tau_{mix}^{\theta_t} \log T_{max} } + \sum_{t=1}^T 8 \log(T_{max}) T_{max} \left( \mathcal{E}_2(t) + 16B^2 \mathcal{E}_{app}^2 \right) } \\
    &\labelrel{\leq}{ineq:actor_bd_6} \wo{(M+L) G_H \sqrt{ T \max_{t \in [T]} \tau_{mix}^{\theta_t} \log T_{max} }} + (2M+L) \sqrt{ 8 \sum_{t=1}^T \log(T_{max}) T_{max} \left( \mathcal{E}_2(t) + 16B^2 \mathcal{E}_{app}^2 \right) },% \label{ineq:actor_bd_7}
\end{align}
where \eqref{ineq:actor_bd_4} follows by the fact that $\Delta_{\max} \leq 2M$ and Jensen's inequality, \eqref{ineq:actor_bd_5} is from Lemma \ref{lemma:31ss}, and \eqref{ineq:actor_bd_6} follows since $\sqrt{a + b} \leq \sqrt{a} + \sqrt{b}$. Furthermore, by the second-order bound of Lemma \ref{lemma:a6ss} we have
\begin{align}
    \sum_{t=1}^T \E \left[ \norm{\nabla J(\theta_t) - h_t^{j_{max}}}^2 \right] &\leq \wo{ T G_H^2 \tau_{mix}^{\theta_t} \frac{\log T_{max}}{T_{max}} } + \sum_{t=1}^T \mathcal{E}_2(t) + T 16B^2 \mathcal{E}_{app}^2.
\end{align}
Combining these expressions and dividing by $T$ completes the proof.

\end{proof}

\end{theorem}

\section{Average Reward Tracking and Critic Error Analyses} \label{sec:rew_and_critic_analysis}
In this section we bound the error arising from the average reward tracking and critic estimation. Combined with the actor gradient norm bound of Section \ref{sec:actor_analysis}, this will complete the analysis of Algorithm \ref{alg:PG_MAG}. Our analysis broadly follows that of \cite{wu2020finite}, with key modifications leveraging our novel MLMC machinery to handle Markovian sampling in a more streamlined manner.

\subsection{Average Reward Tracking Analysis}
The main result of this subsection is the following bound on the average reward tracking error.
\begin{theorem} \label{thm:reward_analysis_app}
Assume $\gamma_t = (1 + t)^{-\nu}, \alpha = \alpha_t' / \sqrt{\sum_{k=1}^t \norm{h_t}^2}$, and $\alpha_t' = (1 + t)^{-\sigma}$, where $0 < \nu < \sigma < 1$. Furthermore, assume $\sup_{s,a} |r(s,a)| \leq R$. Then
\begin{align}
    \frac{1}{T} \sum_{t=1}^T \E \left[ (\eta_t - \eta_t^*)^2 \right] &\leq \bo{T^{\nu - 1}} + \bo{T^{-2(\sigma - \nu)}} \\
    &+ \wo{ \max_{t \in [T]} \tau_{mix}^{\theta_t} \log T_{max}} \bo{T^{-\nu}} \\
    &+ \wo{ \sqrt{ \max_{t \in [T]} \tau_{mix}^{\theta_t} \frac{\log T_{max}}{T_{max}} } }.
\end{align}
\end{theorem}

\begin{proof}
First, recall that the average reward tracking update is given by
\begin{align}
	\eta_{t+1}=\eta_t - \gamma_t f_t, 
\end{align}
where for brevity we set $f_t := f_t^{\text{MLMC}}$. We can rewrite the tracking error term $(	\eta_{t+1}-\eta_{t+1}^*)^2$ as
\begin{align}
	(	\eta_{t+1}-\eta_{t+1}^*)^2 &= (\eta_{t+1}-\eta_{t}^*+\eta_{t}^*-\eta_{t+1}^*)^2
	\\
	&= (\eta_t - \gamma_t f_t-\eta_{t}^*+\eta_{t}^*-\eta_{t+1}^*)^2.
\end{align}
Expanding the squares and regrouping terms yields
\begin{align}
	(	\eta_{t+1}-\eta_{t+1}^*)^2 	&= (	\eta_{t}-\eta_{t}^*)^2 - 2\gamma_t (	\eta_{t}-\eta_{t}^*)f_t + 2(	\eta_{t}-\eta_{t}^*) (\eta_{t}^*-\eta_{t+1}^*) 
	\nonumber
	\\
	& \qquad -2\gamma_t (\eta_{t}^*-\eta_{t+1}^*)f_t+ (\eta_{t}^*-\eta_{t+1}^*)^2 +\gamma_t^2(f_t)^2
	\\
	&= (	\eta_{t}-\eta_{t}^*)^2 - 2\gamma_t (	\eta_{t}-\eta_{t}^*)f_t + 2(	\eta_{t}-\eta_{t}^*) (\eta_{t}^*-\eta_{t+1}^*) 
		\nonumber
	\\
	& \qquad +(\eta_{t}^*-\eta_{t+1}^*-\gamma_tf_t)^2.\label{ref_1}
\end{align}
Next, we utilize the bound $(a+b)^2\leq 2a^2+2b^2$ to upper bound the last term in the right hand side of \eqref{ref_1} to obtain
\begin{align}
	(	\eta_{t+1}-\eta_{t+1}^*)^2 	&\leq  (	\eta_{t}-\eta_{t}^*)^2 - 2\gamma_t (	\eta_{t}-\eta_{t}^*)f_t + 2(	\eta_{t}-\eta_{t}^*) (\eta_{t}^*-\eta_{t+1}^*) 
	\nonumber
	\\
	& \qquad +2(\eta_{t}^*-\eta_{t+1}^*)^2+2(\gamma_tf_t)^2.\label{ref_2}
\end{align}
Now notice that the function whose gradient we are estimating with $f_t$ is simply the strongly convex function $F(\eta_t) = \frac{1}{2} \left( \eta_t - \eta_t^* \right)^2 = \frac{1}{2} \left( \eta_t - J(\theta_t) \right)^2$. Clearly $F'(\eta_t) = \eta_t - J(\theta_t)$ is Lipschitz in $\eta_t$ and $F$ has strong convexity parameter $m_F=1$. Adding and subtracting $2\gamma_t (	\eta_{t}-\eta_{t}^*) F' (\eta_{t})$ in the above expression gives
\begin{align}
	(	\eta_{t+1}-\eta_{t+1}^*)^2 	&\leq  (	\eta_{t}-\eta_{t}^*)^2 - 2\gamma_t (	\eta_{t}-\eta_{t}^*)F' (\eta_{t}) + 2\gamma_t (\eta_{t}-\eta_{t}^*)(F' (\eta_{t})-f_t) + 2(	\eta_{t}-\eta_{t}^*) (\eta_{t}^*-\eta_{t+1}^*) 
	\nonumber
	\\
	& \qquad +2(\eta_{t}^*-\eta_{t+1}^*)^2+2(\gamma_tf_t)^2.\label{ref_3}
\end{align}
From the strong convexity of $F$ with $m_F = 1$, we can write
\begin{align}
	(	\eta_{t+1}-\eta_{t+1}^*)^2 	&\leq  (	\eta_{t}-\eta_{t}^*)^2 - 2\gamma_t (\eta_{t}-\eta_{t}^*)^2+ 2\gamma_t (\eta_{t}-\eta_{t}^*)(F' (\eta_{t})-f_t) + 2(	\eta_{t}-\eta_{t}^*) (\eta_{t}^*-\eta_{t+1}^*) 
	\nonumber
	\\
	&+2(\eta_{t}^*-\eta_{t+1}^*)^2+2(\gamma_tf_t)^2
	\\
	&= (1- 2\gamma_t) (\eta_{t}-\eta_{t}^*)^2+ 2\gamma_t (\eta_{t}-\eta_{t}^*)(F' (\eta_{t})-f_t) + 2(	\eta_{t}-\eta_{t}^*) (\eta_{t}^*-\eta_{t+1}^*) 
	\nonumber
	\\
	&+2(\eta_{t}^*-\eta_{t+1}^*)^2+2(\gamma_tf_t)^2.\label{ref_4}
\end{align}
Taking expectations and summing yields
\begin{align}
	\sum_{t=1}^{T}\E[(\eta_{t}-\eta_{t}^*)^2] 	\leq  & \underbrace{\sum_{t=1}^{T}\frac{1}{2\gamma_t} 	\E[(\eta_{t}-\eta_{t}^*)^2-(\eta_{t}-\eta_{t}^*)^2]}_{I_1}+  \underbrace{\sum_{t=1}^{T}	\E[(\eta_{t}-\eta_{t}^*)(F' (\eta_{t})-f_t)] }_{I_{2}}	\nonumber
	\\
	&+ \underbrace{\sum_{t=1}^{T}\frac{1}{\gamma_t}\E[(\eta_{t}-\eta_{t}^*) (\eta_{t}^*-\eta_{t+1}^*) ]}_{I_{3}}
+ \underbrace{\sum_{t=1}^{T}\frac{1}{\gamma_t}	\E[(\eta_{t}^*-\eta_{t+1}^*)^2]}_{I_4}+ \underbrace{\sum_{t=1}^{T}\gamma_t\E[(f_t)^2]}_{I_5}.\label{ref_5}
\end{align}
We next provide intermediate bounds for all the terms $I_1, I_2, I_3, I_4$ and $I_5$ in the right hand side of \eqref{ref_5}. We will subsequently manipulate these intermediate bounds to obtain the final bound of Theorem \ref{thm:reward_analysis_app}.

%%%%%%%%%%%%%%%%%%%%%%%%%%%%%%%%%%%%%%%

\textbf{Bound on $I_1$:} By rearranging terms in $I_1$, we get
\begin{align}
I_1 &=	\sum_{t=1}^{T}\frac{1}{2\gamma_t} 	\E[(\eta_{t}-\eta_{t}^*)^2-(\eta_{t}-\eta_{t}^*)^2]
\nonumber
\\
&=\frac{1}{2\gamma_1}\E[(\eta_{1}-\eta_{1}^*)^2]+ \sum_{t=2}^{T}\left(\frac{1}{2\gamma_t} -\frac{1}{2\gamma_{t-1}}\right)	\E[(\eta_{t}-\eta_{t}^*)^2]-\frac{1}{2\gamma_T}\E[(\eta_{T+1}-\eta_{T+1}^*)^2]
\\
&\leq \frac{R^2}{\gamma_T}, \label{proof_I_1}
\end{align}
where we use the fact that $(\eta_{t}-\eta_{t}^*)^2\leq 2R^2$.

\textbf{Bound on $I_2$:} For $I_2$, first notice that $\eta_t, \eta_t^*=J(\theta_t)$ are deterministic conditioned on $\mathcal{F}_{t-1}$ from Lemma \ref{lemma:a5}. This means we can rewrite the expectation in $I_2$ as
\begin{equation}
	I_2 = \sum_{t=1}^{T}	\E[\E_{t-1}[(\eta_{t}-\eta_{t}^*)(F' (\eta_{t})-f_t)]] = \sum_{t=1}^{T}	\E[(\eta_{t}-\eta_{t}^*)(F' (\eta_{t})-\E_{t-1}[f_t])], \label{ref_7}
\end{equation}
where $\E_{t-1} [\ldots]$ denotes expectation conditioned on $\mathcal{F}_{t-1}$. From \ref{lemma:31ss} we know that $\E_{t-1}[f_t] = \E_{t-1}[f_t^{j_{\max}}]$, hence we can write the expression in \eqref{ref_7} as 
\begin{align}
	I_2 &= \sum_{t=1}^{T} \E[(\eta_{t}-\eta_{t}^*)(F' (\eta_{t})- \E_{t-1}[f_t^{j_{\max}})]] = \sum_{t=1}^{T} \E[ \E_{t-1}[ (\eta_{t}-\eta_{t}^*)(F' (\eta_{t})- f_t^{j_{\max}})]] \\
    &= \sum_{t=1}^{T} \E[(\eta_{t}-\eta_{t}^*)(F' (\eta_{t}) - f_t^{j_{\max}})]. \label{ref_8}
\end{align}
Taking absolute values, then applying the triangle, Jensen, and Cauchy-Schwarz inequalities, we can upper bound \eqref{ref_8} by
\begin{align}
	|I_2| &= \bigg |\sum_{t=1}^{T}	\E[(\eta_{t}-\eta_{t}^*)(F' (\eta_{t})-f_t^{j_{\max}})] \bigg |
	\leq \sum_{t=1}^{T}	\E\left[\big |(\eta_{t}-\eta_{t}^*)(F' (\eta_{t})-f_t^{j_{\max}})\big |\right]	\nonumber
	\\
	&\leq \sum_{t=1}^{T}	\E\left[\big |(\eta_{t}-\eta_{t}^*)\big |\cdot \big |(F' (\eta_{t})-f_t^{j_{\max}})\big |\right] . \label{ref_9}
\end{align}
We know that $|\eta_{t}-\eta_{t}^*|\leq 2R$ by assumption, implying
\begin{align}
	|I_2| &\leq 2R\sum_{t=1}^{T}	\E\left[\big |(F' (\eta_{t})-f_t^{j_{\max}})\big |\right] . \label{ref_10}
\end{align}
By Lemma \ref{lemma:a6} with $x_t = \eta_t, \nabla L(x_t) = \nabla F(\eta_t)$ and $l(x_t, z_t) = f_t$, and the fact that the Lipschitz constant of $\nabla F(\eta_t)$ is 1, we obtain the following upper bound on $I_2$:
\begin{align}
	|I_2| &\leq 2R\sum_{t=1}^{T} \wo{\sqrt{\tau_{\text{mix}}^{\theta_t} \frac{\log T_{\max}}{T_{\max}}}} . \label{proof_I_2}
\end{align}

\textbf{Bound on $I_3$:} By H\"{o}lder's inequality,
\begin{align}
    |I_3| &= \bigg|\sum_{t=1}^{T}\frac{1}{\gamma_t}\E[(\eta_{t}-\eta_{t}^*) (\eta_{t}^*-\eta_{t+1}^*) ]\bigg | \leq \left( \sum_{t=1}^T \E \left[ (\eta_t - \eta_t^*)^2 \right] \right)^{1/2} \left( \sum_{t=1}^T \frac{1}{\gamma_t^2} \E \left[ (\eta_t^* - \eta_{t+1}^*)^2 \right] \right)^{1/2}.
\end{align}
Notice that $| \eta_t^* - \eta_{t+1}^* | = | J(\theta_t) - J(\theta_{t+1}) | \leq L | \theta_t - \theta_{t+1} | \leq LG_H \alpha_t$ due to the Lipschitz continuity of $J(\theta)$ in $\theta$ and boundedness of $\norm{\nabla J(\theta)}$ from Assumption \ref{assum:pg_objective}. This implies
\begin{equation}
|I_3| \leq \left( \sum_{t=1}^T \E \left[ (\eta_t - \eta_t^*)^2 \right] \right)^{1/2} \left( L^2 G_H^2 \sum_{t=1}^T \frac{\alpha_t^2}{\gamma_t^2} \right)^{1/2}. \label{ref_11}
\end{equation}

\textbf{Bound on $I_4$:} Similarly, due to Assumption \ref{assum:pg_objective} we have
\begin{equation}
I_4 =	\sum_{t=1}^{T}\frac{1}{\gamma_t}	\E[(\eta_{t}^*-\eta_{t+1}^*)^2] \leq  L^2 G_H^2 \sum_{t=1}^{T}\frac{\alpha^2}{\gamma_t}.
\end{equation}

\textbf{Bound on $I_5$:} Finally, by Lemma \ref{lemma:31_app} and taking $G_F = 2R$ without loss of generality, we have 
\begin{equation}
I_5 = \sum_{t=1}^{T}\gamma_t\E[(f_t)^2] \leq \sum_{t=1}^{T}\gamma_t \wo{ R^2 \tau_{\text{mix}}^{\theta_t} \log T_{\max} }. \label{proof_I_5}
\end{equation}

Combining the foregoing and recalling that $\gamma_t = (1+t)^{-\nu}, \alpha_t' = (1+t)^{-\sigma}$, $0 < \nu < \sigma < 1$, and $\alpha_t \leq \alpha_t'$, we get
\begin{align}
	\sum_{t=1}^{T} \E[(\eta_{t}-\eta_{t}^*)^2] 	&\leq  2R^2(1+T)^\nu +  2TR \wo{ \sqrt{\max _{t\in[T]}\tau_{\text{mix}}^{\theta_t} \frac{\log T_{\max}}{T_{\max}}} } \\
	& \qquad + LG_H \left(\sum_{t=1}^{T}\E[(\eta_t-\eta_t^*)^2]\right)^{\frac{1}{2}}\left(\sum_{t=1}^{T}(1+t)^{-2(\sigma-\nu)}\right)^{\frac{1}{2}}	\\
	&\qquad + L^2 G_H^2 \sum_{t=1}^{T} (1+t)^{(\nu-2\sigma)} + \wo{ \max_{t\in[T]}\tau_{\text{mix}}^{\theta_t}\log T_{\max} } \sum_{t=1}^{T}(1+t)^{-\nu} \\
	&\leq  {2R^2(1+T)^\nu} + \left[L^2G_H^2 + \wo{ \max_{t\in[T]}\tau_{\text{mix}}^{\theta_t}\log T_{\max} } \right]\sum_{t=1}^{T}(1+t)^{-\nu} \\
	& \qquad + 2TR \wo{ \sqrt{\max _{t\in[T]}\tau_{\text{mix}}^{\theta_t} \frac{\log T_{\max}}{T_{\max}}} } \\
	& \qquad + \left(\sum_{t=1}^{T}\E[(\eta_t-\eta_t^*)^2]\right)^{\frac{1}{2}}\left( L^2 G_H^2 \sum_{t=1}^{T}(1+t)^{-2(\sigma-\nu)}\right)^{\frac{1}{2}},
\end{align}
where the second inequality follows from the fact that $\nu - 2\sigma < -\nu$.

We now manipulate the foregoing inequality to obtain the desired bound. Define
\begin{align}
    A(T) &= \sum_{t=1}^{T} \E[(\eta_{t}-\eta_{t}^*)^2], \\
    B(T) &= \frac{L^2 G_H^2}{4} \sum_{t=1}^{T}(1+t)^{-2(\sigma-\nu)}, \\
    C(T) &= 2R^2(1+T)^\nu + \left[L^2G_H^2 + \wo{ \max_{t\in[T]}\tau_{\text{mix}}^{\theta_t}\log T_{\max} } \right]\sum_{t=1}^{T}(1+t)^{-\nu} \\
	& \qquad + 2TR \wo{ \sqrt{\max _{t\in[T]}\tau_{\text{mix}}^{\theta_t} \frac{\log T_{\max}}{T_{\max}}} }
\end{align}
We can thus rewrite the foregoing inequality as
\begin{equation}
    A(T) \leq C(T) + 2 \sqrt{A(T)} \sqrt{B(T)}.
\end{equation}
This expression is equivalent to
\begin{equation}
    \left( \sqrt{A(T)} - \sqrt{B(T)} \right)^2 \leq C(T) + B(T),
\end{equation}
which in turn gives the following chain of implications:
\begin{align}
    \left( \sqrt{A(T)} - \sqrt{B(T)} \right)^2 \leq C(T) + B(T) &\Longrightarrow \sqrt{ A(T) } - \sqrt{ B(T) } \leq \sqrt{ C(T) } + \sqrt{ B(T) } \\
    &\Longrightarrow \sqrt{ A(T) } \leq \sqrt{ C(T) } + 2 \sqrt{ B(T) } \\
    &\Longrightarrow A(T) \leq 2C(T) + 4B(T).
\end{align}
As a result, we have shown that
\begin{align}
    \sum_{t=1}^{T} \E[(\eta_{t}-\eta_{t}^*)^2] &\leq 4R^2(1+T)^\nu + \left[2L^2G_H^2 + \wo{ \max_{t\in[T]}\tau_{\text{mix}}^{\theta_t}\log T_{\max} } \right]\sum_{t=1}^{T}(1+t)^{-\nu} \\
	& \qquad + 4TR \wo{ \sqrt{\max _{t\in[T]}\tau_{\text{mix}}^{\theta_t} \frac{\log T_{\max}}{T_{\max}}} } \\
    & \qquad + L^2 G_H^2 \sum_{t=1}^{T}(1+t)^{-2(\sigma-\nu)}.
\end{align}
Using the bound $\sum_{t=1}^T (1 + t)^{-\xi} \leq \int_0^{t+1} x^{-\xi} dx = (t+1)^{1-\xi} / (1-\xi)$, this implies
\begin{align}
    \sum_{t=1}^{T} \E[(\eta_{t}-\eta_{t}^*)^2] &\leq O(T^{\nu}) + \wo{ \max_{t\in[T]}\tau_{\text{mix}}^{\theta_t}\log T_{\max} } O(T^{1-\nu}) + O(T^{1-2(\sigma - \nu)}) \\
    & \qquad + T \wo{ \sqrt{\max _{t\in[T]}\tau_{\text{mix}}^{\theta_t} \frac{\log T_{\max}}{T_{\max}}} }
\end{align}
Dividing by $T$ completes the proof.
\end{proof}

Notice that, for $\sigma=0.75$ and $\nu=0.5$, this result becomes
\begin{align}
	\frac{1}{T}\sum_{t=1}^{T}\E[(\eta_{t}-\eta_{t}^*)^2] 
	&\leq  \wo{ \max_{t\in[T]}\tau_{\text{mix}}^{\theta_t}\log T_{\max} } O \left(\frac{1}{\sqrt{T}}\right) + \wo{ \sqrt{\max _{t\in[T]}\tau_{\text{mix}}^{\theta_t} \frac{\log T_{\max}}{T_{\max}}} }.\label{final_6}
\end{align}
\subsection{Critic Error Analysis}

In this subsection we provide a bound on the critic estimation error term $\frac{1}{T} \sum_{t=1}^T \E \left[ \norm{ \omega_t - \omega_t^* }^2 \right]$ appearing in the main actor analysis bound in Theorem \ref{thm:actor_bd_app}. To get started, we recall some facts about the TD(0) algorithm \cite{sutton1988}. As discussed in Ch. 9 of \cite{sutton2018reinforcement}, for a fixed policy parameter, $\theta$, TD(0) with linear function approximation will converge to the minimum of the mean squared projected Bellman error (MSPBE), which satisfies
\begin{align}
    & \hspace{1.5cm} A_{\theta} \omega = b_{\theta}, \\
    A_{\theta} &= \E_{s \sim \mu_{\theta}, a \sim \pi_{\theta}, s' \sim p(\cdot | s, a)} \left[ \phi(s) (\phi(s) - \phi(s'))^T \right], \\
    b_{\theta} &= \E_{s \sim \mu_{\theta}, a \sim \pi_{\theta}} \left[ ( r(s, a) - J(\theta) ) \phi(s) \right].
\end{align}
The target critic parameter $\omega_t^*$ at iteration $t$ of our Algorithm \ref{alg:PG_MAG} is thus given by $\omega_t^* = A_{\theta_t}^{-1} b_{\theta_t}$. From the definition of $g_t^{MLMC}$, the critic update $\omega_{t+1} = \omega_t + \beta_t g_t^{MLMC}$ is clearly an attempt to use an MLMC estimator to approximately perform the ideal update $\omega_{t+1} = \omega_t + \beta_t (b_{\theta_t} - A_{\theta_t} \omega_t)$. We can thus view $\nabla G(\omega_t) = b_{\theta_t} - A_{\theta_t} \omega_t$ as the gradient of the true critic objective $G(\omega_t)$ corresponding to using least squares minimization to solve the equation $A_{\theta} \omega = b_{\theta}$.

Our task in this section is to characterize the average error that arises when using critic parameters $\{ \omega_t \}$  generated by Algorithm \ref{alg:PG_MAG} to track the ideal parameters $\{ \omega_t^* \}$. Before we provide the main result of this section, we need three useful lemmas and an assumption. The first result ensures that the optimal critic parameter is Lipschitz in $\theta$.
\begin{lemma} \label{lemma:critic_Lipschitz_app}
    Define $P_{\theta}(s' | s) = \int_{\mathcal{A}} p(s' | s, a) \pi_{\theta}(a | s) da$, for each $\theta$. {Assume that, for all $\theta$, the ergodicity coefficient $\kappa(P_{\theta})$ of $P_{\theta}$ satisfies $\kappa(P_{\theta}) < 1$.} Then there exists $L_{\omega}$ such that, for all $\theta, \theta'$, $\omega^*(\theta) = A_{\theta}^{-1} b_{\theta}$ and $\omega^*(\theta') = A_{\theta'}^{-1} b_{\theta'}$ satisfy $\norm{\omega^*(\theta) - \omega^*(\theta')} \leq L_{\omega} \norm{\theta - \theta'}$.
\end{lemma}
\begin{proof}
    The result follows by applying the same reasoning as that for Lemma A.3 in \cite{zou2019finite} to the bound from Theorem 3.3 in \cite{mitrophanov2005sensitivity}.
\end{proof}

The next result is an extension of Lemma \ref{lemma:a6} to our MLMC critic gradient estimator.
\begin{lemma} \label{lemma:a6sss}
    Assume $\norm{\nabla G(\omega)} \leq G_G$, for all $\omega$ such that $\norm{\omega} \leq R_{\omega}$. Define $D = \sup_s \norm{\phi(s)}$. Fix $T_{max} \in \mathbb{N}, \theta_t$ measurable with respect to $\mathcal{F}_{t-1}$, and let $K = \tau_{max}^{\theta_t} \ceil{2 T_{max}}$. Define $g_t^N = \frac{1}{N} \sum_{i=1}^N \delta_t^i \phi(s_t^i)$, for $N \in \left[ T_{max} \right]$, where $\delta_t^i = r_t^i - \eta_t + (\phi(s_t^{i+1}) - \phi(s_t^i))^T \omega_t$. Then, for all $N \in [T_{max}]$,
    \begin{align}
        \E \left[ \bignorm{ g_t^N - \nabla G(\omega_t) } \right] &\leq O \left( G_G \sqrt{ \log KN } \sqrt{\frac{K}{N}} \right) + D \E \left[ | \eta_t - \eta_t^* | \right], \\
        \E \left[ \bignorm{ g_t^N - \nabla G(\omega_t) }^2 \right] &\leq O \left( G_G^2 \log(KN) \frac{K}{N} \right) + D^2 \E \left[ ( \eta_t - \eta_t^* )^2 \right].
    \end{align}
\end{lemma}
\begin{proof}
    Define
    \begin{align}
        \delta_t^{i, \eta} &= r_t^i - \eta_t^* + (\phi(s_t^{i+1}) - \phi(s_t^i))^T \omega_t, \\
        g_t^{N, \eta} &= \frac{1}{N} \sum_{i=1}^{N} \delta_t^{i, \eta} \phi(s_t^i).
    \end{align}
    Clearly
    \begin{align}
        \bignorm{ g_t^N - \nabla G(\omega_t) } &\leq \bignorm{ g_t^N - g_t^{N, \eta} } + \bignorm{ g_t^{N, \eta} - \nabla G(\omega_t) } \\
        &= \bignorm{ \frac{1}{N} \sum_{i=1}^N \delta_t^i \phi(s_t^i) - \delta_t^{i, \eta} \phi(s_t^i) } + \bignorm{ \frac{1}{N} \sum_{i=1}^N \delta_t^{i, \eta} \phi(s_t^i) - \nabla G(\omega_t) }.
    \end{align}
    Notice that the first term can be bounded by $D | \eta_t - \eta_t^* |$ and that Lemma \ref{lemma:a6} applies to the second term. The remainder of the proof is analogous to that of Lemma \ref{lemma:a6ss}.
\end{proof}

Next, we need a critic version of Lemma \ref{lemma:31_app}.
\begin{lemma} \label{lemma:31sss}
    Let $j_{max} = \floor{ \log T_{max} }$ and fix $\theta_t$ measurable w.r.t. $\mathcal{F}_{t-1}$. Assume $T_{max} \geq \tau_{mix}^{\theta_t}$ and $\norm{ \nabla G(\omega) } \leq G_G$, for all $\omega$ such that $\norm{\omega} \leq R_{\omega}$. Then
    \begin{align}
        \E_{t-1} \left[ g_t \right] &= \E_{t-1} \left[ g_t^{j_{max}} \right] \\
        \E \left[ \norm{ g_t }^2 \right] &\leq \wo{ G_G^2 \tau_{mix}^{\theta_t} \log T_{max} } + 8 \log(T_{max}) T_{max} D^2 \E \left[ (\eta_t - \eta_t^*)^2 \right].
    \end{align}
\end{lemma}
\begin{proof}
    The claim follows from Lemma \ref{lemma:a6sss} by the same argument as that used in the proof of Lemma \ref{lemma:31ss}.
\end{proof}

We now provide the main result of this section. The analysis is a modification of that used for the average reward tracking setting.
\begin{theorem} \label{thm:critic_bd_app}
Assume $\beta_t = (1 + t)^{-\nu}, \alpha_t = \alpha_t' / \sqrt{\sum_{k=1}^t \norm{h_t}^2}$, and $\alpha_t' = (1 + t)^{-\sigma}$, where $0 < \nu < \sigma < 1$. Assume without loss of generality that $\alpha_t \leq \alpha_t'$, for all $t$. Furthermore, assume that Assumptions \ref{assum:omega_projection} and \ref{assum:positive_definite} hold. Then
\begin{align}
    \frac{1}{T} \sum_{t=1}^T \E \left[ \norm{\omega_t - \omega_t^*}^2 \right] &\leq \bo{T^{\nu - 1}} + \bo{T^{-2(\sigma - \nu)}} \\
    &\qquad + \wo{ \max_{t \in [T]} \tau_{mix}^{\theta_t} \log T_{max}} \bo{T^{-\nu}} \\
    %
    % &\qquad + \wo{ \sqrt{ \max_{t \in [T]} \tau_{mix}^{\theta_t} \frac{\log T_{max}}{T_{max}} } }.
    %
    &\qquad + \wo{ \max_{t \in [T]} \tau_{mix}^{\theta_t} \frac{\log T_{max}}{T_{max}} }.
\end{align}
\end{theorem}
\begin{proof}
    By Assumption \ref{assum:positive_definite} and the fact that $\nabla^2 G(\omega) = -A_{\theta}$, $G(\omega)$ is strongly concave. Let $m$ denote its strong concavity parameter, so that $\langle \nabla G(\omega) - \nabla G(\omega'), \omega - \omega' \rangle \leq - m \norm{ \omega - \omega' }^2$, for all $\omega, \omega'$. Recall that $\omega_{t+1} = \Pi_{R_{\omega}} \left( \omega_t + \beta g_t \right)$, where we use $g_t = g_t^{MLMC}$ for brevity. We have
    \begin{align}
        \norm{ \omega_{t+1} - \omega_{t+1}^* }^2 &= \norm{ \Pi_{R_{\omega}} \left( \omega_t + \beta g_t \right) - \omega_{t+1}^* }^2 \leq \norm{ \omega_t + \beta g_t - \omega_{t+1}^* }^2,
    \end{align}
    where the inequality holds since $\norm{\omega_{t+1}^*} \leq R_{\omega}$ by definition, so projection can only reduce the distance. Furthermore,
    \begin{align}
        \norm{ w_{t+1}-w_{t+1}^* }^2 &\leq \norm{ w_t - \beta_t h_t-w_{t}^*+w_{t}^*-w_{t+1}^* }^2 \\
        &= \norm{ \omega_t - \omega_t^* }^2 + 2 \beta_t \langle \omega_t - \omega_t^*, g_t \rangle + 2 \langle \omega_t - \omega_t^*, \omega_t^* - \omega_{t+1}^* \rangle \\
        & \qquad + 2 \beta_t \langle \omega_t^* - \omega_{t+1}^*, g_t \rangle + \norm{ \omega_t^* - \omega_{t+1}^* }^2 + \beta_t^2 \norm{ h_t }^2 \\
        &\labelrel{\leq}{ineq:critic_1} \norm{ \omega_t - \omega_t^* }^2 + 2 \beta_t \langle \omega_t - \omega_t^*, g_t \rangle + 2 \langle \omega_t - \omega_t^*, \omega_t^* - \omega_{t+1}^* \rangle \\
        & \qquad + 2 \norm{ \omega_t^* - \omega_{t+1}^* }^2 + 2 \beta_t^2 \norm{ h_t }^2 \\
        &= \norm{ \omega_t - \omega_t^* }^2 + 2 \beta_t \langle \omega_t - \omega_t^*, \nabla G(\omega_t) \rangle + 2 \beta_t \langle \omega_t - \omega_t^*, g_t - \nabla G(\omega_t) \rangle \\
        & \qquad + 2 \langle \omega_t - \omega_t^*, \omega_t^* - \omega_{t+1}^* \rangle + 2 \norm{ \omega_t^* - \omega_{t+1}^* }^2 + 2 \beta_t^2 \norm{ h_t }^2 \\
        &\labelrel{\leq}{ineq:critic_2} (1 - 2 m \beta_t) \norm{ \omega_t - \omega_t^* }^2 + 2 \beta_t \langle \omega_t - \omega_t^*, g_t - \nabla G(\omega_t) \rangle \\
        & \qquad + 2 \langle \omega_t - \omega_t^*, \omega_t^* - \omega_{t+1}^* \rangle + 2 \norm{ \omega_t^* - \omega_{t+1}^* }^2 + 2 \beta_t^2 \norm{ h_t }^2,
    \end{align}
    where \eqref{ineq:critic_1} follows from completing the square with the last three terms and the fact that $(a + b)^2 \leq 2a^2 + 2b^2$, and \eqref{ineq:critic_2} follows from the strong concavity of $G(\omega)$.

    Rearranging, dividing by $2m\beta_t$, taking expectations, and summing yields
    \begin{align}
        \sum_{t=1}^{T}\E[\|w_{t}-w_{t}^*\|^2] 	\leq  & \underbrace{\sum_{t=1}^{T}\frac{1}{2 m \beta_t} 	\E[\|w_{t}-w_{t}^*\|^2-\|w_{t}-w_{t}^*\|^2]}_{M_1}+  \underbrace{\sum_{t=1}^{T}	\frac{1}{m} \E[ \langle w_{t}-w_{t}^*, \nabla G(w_{t}) - g_t\rangle] }_{M_{2}}	\nonumber
        \\
        &+ \underbrace{\sum_{t=1}^{T}\frac{1}{m \beta_t}\E[\langle 	w_{t}-w_{t}^*, w_{t}^*-w_{t+1}^*\rangle]}_{M_{3}}
        + \underbrace{\sum_{t=1}^{T}\frac{1}{m \beta_t}	\E[\|w_{t}^*-w_{t+1}^*\|^2]}_{M_4}+ \underbrace{\sum_{t=1}^{T} \frac{\beta_t}{m} \E[\|g_t\|^2]}_{M_5}.\label{ref_5_critic}
    \end{align}

    As in the proof of Theorem \ref{thm:reward_analysis_app}, we first provide intermediate bounds on $M_1, M_2, M_3, M_4, M_5$, then manipulate the resulting expressions to obtain the desired, final bound on the critic error. With the exception of $M_2$, the intermediate bounds follow by the same reasoning as their counterparts in Theorem \ref{thm:reward_analysis_app}.

    \textbf{Bound for $M_1$:} By the same reasoning as for $I_1$,
    \begin{equation}
        M_1 \leq \frac{2 R_{\omega}^2}{m \beta_t}.
    \end{equation}
    \textbf{Bound for $M_2$:} Since $\omega_t, \omega_t^*$ are deterministic given $\mathcal{F}_{t-1}$, by the law of total expectation and Lemma \ref{lemma:31sss} we have
    \begin{equation}
        M_2 = \sum_{t=1}^T \frac{1}{m} \E \left[ \langle \omega_t - \omega_t^*, g_t^{j_{max}} - \nabla G(\omega_t) \rangle \right].
    \end{equation}
    Furthermore,
    \begin{align}
        | M_2 | &\labelrel{\leq}{ineq:critic_3} \sum_{t=1}^T \frac{1}{m} \E \left[ \norm{ \omega_t - \omega_t^* } \cdot \norm{ g_t^{j_{max}} - \nabla G(\omega_t) } \right] \\
        &\labelrel{\leq}{ineq:critic_4} \sum_{t=1}^T \frac{1}{m} \left( \E \left[ \norm{ \omega_t - \omega_t^* }^2 \right] \right)^{1/2} \left( \E \left[ \norm{ g_t^{j_{max}} - \nabla G(\omega_t) }^2 \right] \right)^{1/2} \\
        &\labelrel{\leq}{ineq:critic_5} \left( \frac{1}{m^2} \sum_{t=1}^T \E \left[ \norm{ \omega_t - \omega_t^* }^2 \right] \right)^{1/2} \left( \sum_{t=1}^T \E \left[ \norm{ g_t^{j_{max}} - \nabla G(\omega_t) }^2 \right] \right)^{1/2} \\
        &\labelrel{\leq}{ineq:critic_6} \left( \frac{1}{m^2} \sum_{t=1}^T \E \left[ \norm{ \omega_t - \omega_t^* }^2 \right] \right)^{1/2} \left( T \wo{ G_G^2 \max_{t \in [T]} \tau_{mix}^{\theta_t} \frac{\log T_{max}}{T_{max}} } + D^2 \sum_{t=1}^T \E \left[ \norm{ \eta_t - \eta_t^* }^2 \right] \right)^{1/2},
    \end{align}
    where \eqref{ineq:critic_3} follows by applying the triangle, Jensen's, and Cauchy-Schwarz inequalities, \eqref{ineq:critic_4} and \eqref{ineq:critic_5} follow from H\"{o}lder's inequality, and \eqref{ineq:critic_6} results from applying Lemma \ref{lemma:a6sss}.
    
    \textbf{Bound for $M_3$:} Since $\omega^*(\theta)$ is $L_{\omega}$-Lipschitz in $\theta$ by Lemma \ref{lemma:critic_Lipschitz_app}, we have $\norm{ \omega_t^* - \omega_{t+1}^*} \leq L_{\omega} \norm{ \theta_t - \theta_{t+1} } \leq L_{\omega} G_H \alpha_t$, where we recall that $\sup_{\theta} \norm{ \nabla J(\theta) } \leq G_H$. Thus, by reasoning analogous to $I_3$,
    \begin{align}
        | M_3 | &\leq \left( \sum_{t=1}^T \E \left[ \norm{ \omega_t - \omega_t^* }^2 \right] \right)^{1/2} \left( \frac{L_{\omega}^2 G_H^2}{m^2} \sum_{t=1}^T \frac{\alpha_t^2}{\beta_t^2} \right)^{1/2}.
    \end{align}

    \textbf{Bound for $M_4$:} Similarly,
    \begin{align}
        M_4 &\leq \frac{L_{\omega}^2 G_H^2}{m} \sum_{k=1}^T \frac{\alpha_t^2}{\beta_t}.
    \end{align}

    \textbf{Bound for $M_5$:} Finally, by Lemma \ref{lemma:31sss} and the fact that $|\eta_t| \leq R$, for all $t$,
    \begin{align}
        M_5 &\leq \sum_{t=1}^T \frac{\beta_t}{m} \left[ \wo{ G_H^2 \tau_{mix}^{\theta_t} \log T_{max} } + 8 D^2 \log(T_{max}) T_{max} \E \left[ (\eta_t - \eta_t^*)^2 \right] \right] \\
        &\leq \left[ \wo{ G_H^2 \tau_{mix}^{\theta_t} \log T_{max} } + 16 D^2 R^2 \log(T_{max}) T_{max} \right] \sum_{k=1}^T \frac{\beta_t}{m}.
    \end{align}

    Combining the foregoing and recalling the definitions of $\beta_t, \alpha_t, \alpha_t'$, we have
    \begin{align}
        \sum_{t=1}^T \E &\left[ \norm{ \omega_t - \omega_t^* }^2 \right] \leq \frac{2 R_{\omega}}{m} (1 + t)^{\nu} \\
        & \qquad + \left( \frac{1}{m^2} \sum_{t=1}^T \E \left[ \norm{ \omega_t - \omega_t^* }^2 \right] \right)^{1/2} \left( T \wo{ G_G^2 \max_{t \in [T]} \tau_{mix}^{\theta_t} \frac{\log T_{max}}{T_{max}} } + D^2 \sum_{t=1}^T \E \left[ \norm{ \eta_t - \eta_t^* }^2 \right] \right)^{1/2} \\
        & \qquad + \left( \sum_{t=1}^T \E \left[ \norm{ \omega_t - \omega_t^* }^2 \right] \right)^{1/2} \left( \frac{L_{\omega}^2 G_H^2}{m^2} \sum_{t=1}^T (1+t)^{-2(\sigma - \nu)} \right)^{1/2} \\
        & \qquad + \frac{L_{\omega}^2 G_H^2}{m} \sum_{k=1}^T (1+t)^{\nu - 2\sigma} \\
        & \qquad + \wo{ G_H^2 \max_{t \in [T]} \tau_{mix}^{\theta_t} \log (T_{max}) T_{max}} \sum_{t=1}^T (1+t)^{-\nu}.
    \end{align}

    Define
    \begin{align}
        Z(T) &= \sum_{t=1}^T \E \left[ \norm{ \omega_t - \omega_t^* }^2 \right], \\
        F(T) &= \frac{L_{\omega}^2 G_H^2}{4m^2} \sum_{t=1}^T (1+t)^{-2(\sigma - \nu)}, \\
        G(T) &= \frac{1}{16m} \left[ T \wo{ G_G^2 \max_{t \in [T]} \tau_{mix}^{\theta_t} \frac{\log T_{max}}{T_{max}} } + D^2 \sum_{t=1}^T \E \left[ \norm{ \eta_t - \eta_t^* }^2 \right] \right], \\
        A(T) &= \frac{2 R_{\omega}}{m} (1 + t)^{\nu} + \frac{L_{\omega}^2 G_H^2}{m} \sum_{k=1}^T (1+t)^{\nu - 2\sigma} + \wo{ G_H^2 \max_{t \in [T]} \tau_{mix}^{\theta_t} \log (T_{max}) T_{max}} \sum_{t=1}^T (1+t)^{-\nu}.
    \end{align}

    The previous inequality is thus the same as
    \begin{align}
        Z(T) \leq A(T) + 2 \sqrt{ Z(T) } \sqrt{ F(T) } + 2 \sqrt{ Z(T) } \sqrt{ G(T) },
    \end{align}
    which is in turn equivalent to
    \begin{align}
        \left( \sqrt{ Z(T) } - \sqrt{ F(T) } - \sqrt{ G(T) } \right)^2 \leq A(T) + \left( \sqrt{ F(T) } + \sqrt{ G(T) } \right)^2.
    \end{align}
    This yields
    \begin{align}
        \sqrt{ Z(T) } - \sqrt{ F(T) } - \sqrt{ G(T) } &\leq \left( A(T) + \left( \sqrt{ F(T) } + \sqrt{ G(T) } \right)^2 \right)^{1/2} \\
        &\leq \sqrt{ A(T) } + \sqrt{ F(T) } + \sqrt{ G(T) },
    \end{align}
    whence
    \begin{align}
        \sqrt{ Z(T) } \leq \sqrt{ A(T) } + 2 \sqrt{ F(T) } + 2 \sqrt{ G(T) }
    \end{align}
    and thus
    \begin{align}
        Z(T) &\leq 2 A(T) + 2 \left( 2 \sqrt{ F(T) } + 2 \sqrt{ G(T) } \right)^2 \\
        &\leq 2 A(T) + 16 F(T) + 16 G(T).
    \end{align}

    Noticing that $2 A(T) + 16 F(T) = \bo{T^{\nu}} + \bo{ T^{1 + \nu - 2\sigma}} + \bo{T^{1-\nu}}$ and using the bound $\sum_{t=1}^T (1+t)^{-\xi} \leq (1 + t)^{1-\xi} / (1 - \xi)$, we have
    \begin{align}
        \sum_{t=1}^T \E \left[ \norm{ \omega_t - \omega_t^* }^2 \right] &\leq \frac{1}{m} \left[ T \wo{ G_G^2 \max_{t \in [T]} \tau_{mix}^{\theta_t} \frac{\log T_{max}}{T_{max}} } + D^2 \sum_{t=1}^T \E \left[ \norm{ \eta_t - \eta_t^* }^2 \right] \right] \\
        & \qquad + \bo{T^{\nu}} + \bo{ T^{1 + \nu - 2\sigma}} + \bo{T^{1-\nu}}.
    \end{align}
    Dividing by $T$, combining with Theorem \ref{thm:reward_analysis_app}, and absorbing constants into the order notation finishes the proof.

\end{proof}

\section{Proof of Theorem \ref{thm:convergence_rate}}\label{prof_thm:convergence_rate}
\begin{proof}
From the statement of Theorems \ref{thm:actor_bd} and \ref{thm:critic_analysis_main_body}, we have
\begin{align}\label{actor_adagrad_appendix}
    \frac{1}{T} \sum_{t=1}^T \mathbb{E} & \left[ \bignorm{ \nabla J(\theta_t) }^2 \right] \leq \bo{\frac{1}{\sqrt{T}} } + \bo{ \frac{1}{T} \sum_{t=1}^T \mathcal{E}(t) } + \wo{ \sqrt{\max_{t \in [T]} \tau_{mix}^{\theta_t} \frac{\log T_{\max}}{T_{\max}}} } + \bo{ \mathcal{E}_{app} },
\end{align}
and 
\begin{align}
    \frac{1}{T} \sum_{t=1}^T  \mathcal{E}(t) \leq &\bo{T^{\nu - 1}} + \bo{T^{-2(\sigma - \nu)}} + \wo{ \max_{t \in [T]} \tau_{mix}^{\theta_t} \log T_{\max}} \bo{T^{-\nu}} + \wo{ \sqrt{ \max_{t \in [T]} \tau_{mix}^{\theta_t} \frac{\log T_{\max}}{T_{\max}} } }. \label{ineq:critic_1_appendix}
\end{align}
utilizing the upper bound in \eqref{ineq:critic_1_appendix} into the right hand side of \eqref{actor_adagrad_appendix}, we get
\begin{align}\label{actor_adagrad_appendix2}
    \frac{1}{T} \sum_{t=1}^T \mathbb{E}  \left[ \bignorm{ \nabla J(\theta_t) }^2 \right] \leq & \bo{\frac{1}{\sqrt{T}} }+\bo{T^{\nu - 1}} + \bo{T^{-2(\sigma - \nu)}} + \wo{ \max_{t \in [T]} \tau_{mix}^{\theta_t} \log T_{\max}} \bo{T^{-\nu}}  
    \nonumber
    \\
    &+ \wo{ \sqrt{\max_{t \in [T]} \tau_{mix}^{\theta_t} \frac{\log T_{\max}}{T_{\max}}} } + \bo{ \mathcal{E}_{app} }.
\end{align}
For the selection $\nu=0.5$ and $\sigma=0.75$ (which satisfies the constraint that $0<\nu<\sigma<1$), we obtain
\begin{align}\label{actor_adagrad_appendix3}
    \frac{1}{T} \sum_{t=1}^T \mathbb{E}  \left[ \bignorm{ \nabla J(\theta_t) }^2 \right] \leq & \bo{\frac{1}{\sqrt{T}} }+\bo{\frac{1}{\sqrt{T}} } + \bo{\frac{1}{\sqrt{T}} } + \wo{ \max_{t \in [T]} \tau_{mix}^{\theta_t} \log T_{\max}} \bo{\frac{1}{\sqrt{T}}}  
    \nonumber
    \\
    &+ \wo{ \sqrt{\max_{t \in [T]} \tau_{mix}^{\theta_t} \frac{\log T_{\max}}{T_{\max}}} } + \bo{ \mathcal{E}_{app} }.
\end{align}
Therefore, after further simplification, we can write
\begin{align}\label{actor_adagrad_appendix4}
    \frac{1}{T} \sum_{t=1}^T \mathbb{E}  \left[ \bignorm{ \nabla J(\theta_t) }^2 \right] \leq & \wo{ \max_{t \in [T]} \tau_{mix}^{\theta_t} \log T_{\max}} \bo{\frac{1}{\sqrt{T}}}  
 + \wo{ \sqrt{\max_{t \in [T]} \tau_{mix}^{\theta_t} \frac{\log T_{\max}}{T_{\max}}} } + \bo{ \mathcal{E}_{app} }.
\end{align}
 completes the proof. 

\end{proof}

\section{Hyperparametrs for the Experiments}\label{experiments_details}
We list all the hyperparameters in Table \ref{table_1} here.

\begin{table*}[h]
\centering
%\small
%	\vspace{0.05cm}
	\caption{ This table compares the hyperparameters and performance between the four experiments, each run for five trials. From the table, we see that given the same learning rates, environment, and the number of samples, MAC and Vanilla AC converge to the same reward value.  }
 \vspace{2mm}
%\begin{threeparttable}
	\begin{tabular}{|c|c|c|c|c|c|c|c|c|}
		\hline
	\multirow{2}{*}{Method} & \multicolumn{3}{c|}{Learning Rate}  & Grid Size & $T_{\max}$ & Samples  & Limiting  & Limiting Policy
	\\ 
	\cline{2-4}
	 ~ & ~~Actor~~ & Critic & ~  Reward Estimator &  ~& &  Processed & Mean Reward & Gradient Norm
	 \\
	 \hline{}
		MAC  & $.01$\ & $.01$\ & $.01$ & $6 \times 6$ & $8$ & $3 \cdot 10^6$& $0.4$ & 0\\ \cline{1-9}
	  Vanilla AC & $.01$\ & $.01$\ & $.01$ & $6 \times 6$ & $3$ & $3 \cdot 10^6$& $0.4$ & 0\\ \cline{1-9}
	         MAC  & $.005$\ & $.005$\ & $.005$ & $10 \times 10$ & $16$ & $4 \cdot 10^6$& $0.5$ & 0\\ \cline{1-9}
Vanilla AC & $.005$\ & $.005$\ & $.005$ & $10 \times 10$ & $4$ &$4 \cdot 10^6$& $0.5$ & 0\\ \cline{1-9}
%& \textbf{No} 

% 		\\ \cline{1-5}
% 	         \multirow{2}{*}{Natural Actor-Critic} & (Wang et al., 2019) \cite{wang2019neural} & i.i.d.\ & i.i.d.\ & $\mathcal{O}(\epsilon^{-4})$ \\ \cline{2-5}
% 		~ & \cellcolor{blue!15}{\textcolor{red}{This paper}} & \cellcolor{blue!15}{\textcolor{red}{Markovian}} & \cellcolor{blue!15}{\textcolor{red}{Markovian}} & \cellcolor{blue!15}{\textcolor{red}{$\mathcal{O}(\epsilon^{-3})$}} \\ \cline{2-5}
	\end{tabular}
%
%
%\vspace{-3mm}
\label{table_1}
\end{table*}

\end{document}